\documentclass[twoside]{article}

\PassOptionsToPackage{table}{xcolor}
\usepackage[accepted]{aistats2025}

\usepackage{graphicx} 
\usepackage{amsmath}
\usepackage{amssymb}
\usepackage{amsfonts}
\usepackage{amsthm}
\usepackage{tikz}
\usepackage{caption}
\usepackage[round]{natbib}
\usepackage{hyperref}
\usepackage{xr-hyper}
\usepackage{paralist}
\usepackage{enumitem}
\usepackage[capitalize]{cleveref}
\usepackage{physics}
\usepackage{algpseudocode}
\usepackage{algorithm}

\newcommand*{\boldone}{\text{\usefont{U}{bbold}{m}{n}1}}
\newcommand{\boldee}{\mathbb{E}}

\newtheorem{assumption}{Assumption}

\newtheorem{lemma}{Lemma}
\newtheorem{definition}{Definition}
\newtheorem{theorem}{Theorem}

\newcommand{\tildep}{\widetilde{p}}
\newcommand{\bfD}{\mathbf{D}}
\newcommand{\calD}{\mathcal{D}}
\newcommand{\btheta}{\pmb{\theta}}
\newcommand{\bxi}{\pmb{\xi}}
\newcommand{\calC}{\mathcal{C}}
\newcommand{\tildecalC}{\widetilde{\calC}}
\newcommand{\tildeCoverage}{\tildecalC_{\alpha}}
\newcommand{\bthetaRef}{\btheta_{\text{ref}}}
\newcommand{\bphi}{\pmb{\phi}}

\newcommand{\qVI}{q_{\text{VI}}}

\newcommand{\bphis}{\bphi^{*}}

\newcommand{\tildebg}{\widetilde{\pmb{g}}}
\newcommand{\clip}[1]{\mathrm{clip}\left(#1\right)}

\usepackage{array,longtable}

\begin{document}

\runningtitle{Noise-Aware Differentially Private Variational Inference}
\twocolumn[

\aistatstitle{Noise-Aware Differentially Private Variational Inference\footnotemark[1]}

\aistatsauthor{ Talal Alrawajfeh \And Joonas J\"{a}lk\"{o} \And  Antti Honkela }

\aistatsaddress{ University of Helsinki \\ \texttt{talal.alrawajfeh@helsinki.fi} \And  University of Helsinki \\ \texttt{joonas.jalko@helsinki.fi} \And University of Helsinki \\ 
\texttt{antti.honkela@helsinki.fi}} ]

\footnotetext[1]{Slight typographical errors in the published paper have been corrected in this arXiv version, without changing its results or conclusions. Details of what was corrected are listed in the \hyperref[app:corrections]{table at the end of this document}.}

\begin{abstract}
   Differential privacy (DP) provides robust privacy guarantees for statistical inference, but this can lead to unreliable results and biases in downstream applications. While several noise-aware approaches have been proposed which integrate DP perturbation into the inference, they are limited to specific types of simple probabilistic models.
   In this work, we propose a novel method for noise-aware approximate Bayesian inference based on stochastic gradient variational inference which can also be applied to high-dimensional and non-conjugate models. We also propose a more accurate evaluation method for noise-aware posteriors. 
   Empirically, our inference method has similar performance to existing methods in the domain where they are applicable. 
   Outside this domain, we obtain accurate coverages on high-dimensional Bayesian linear regression and well-calibrated predictive probabilities on Bayesian logistic regression with the UCI Adult dataset.
\end{abstract}

\section{INTRODUCTION}

When applying Bayesian inference on sensitive data, one needs to consider the
privacy risks the released results might pose. Multiple methods combining the
state-of-the-art privacy paradigm differential privacy (DP) \citep{DworkMNS06}
with Bayesian inference have been proposed in the past. These include methods that are based on, e.g., Markov chain Monte Carlo (MCMC)
\citep{WangFS15, HeikkilaJDH19} and variational inference (VI) \citep{dpviinf, dpvim}.
Many of these methods do not consider how the additional noise due to DP affects the
learned posteriors, which might lead to poor uncertainty quantification. As a
solution, multiple works trying to explicitly model the DP noise have since been
proposed \citep{BernsteinS18, BernsteinS19, JuAGR22}, but these are limited to simple or small problems.

In this work, we propose a highly scalable noise-aware DP Bayesian
inference approach that is applicable to a broad range of models. Our method
uses DP variational inference (DPVI) \citep{dpviinf, dpvim}
to obtain noisy gradients from an approximate Bayesian inference problem. 
Modelling the DP noise in the gradients, we form a probabilistic model connecting
the noisy gradients to an optimal solution of the VI problem. Finally, learning the
noise-aware posterior for the optimum, we can form a noise-aware posterior for the original Bayesian inference problem.

\textbf{Related work:}
Several methods have been proposed to take the DP induced noise into account in
private Bayesian inference. Earlier works include sufficient statistic based
models such as noise-aware exponential family models \citep{BernsteinS18} and
Bayesian linear regression \citep{BernsteinS19}. In these works, the data is
released through noisy sufficient statistics, and the perturbation noise is
included as a part of the data generation process. Using approximate sufficient
statistics, \cite{KulkarniJKKH21} extended the noise-awareness for Bayesian
generalized linear models \citep{PASSGLMHuggins2017}. \cite{Gong22} proposed a general summary statistics--based approach using approximate Bayesian computation (ABC). \cite{JuAGR22}
proposed a Gibbs sampling based approach, where the latent confidential data is
augmented into the inference model. Their approach is applicable for a wide
range of models; however, the cost of the inference process scales by the number of
samples in the data. \cite{XiongJZ23} proposed a method which fits a normalizing
flow as a surrogate for the true posterior, by iteratively drawing the model
parameters from the surrogate, drawing the latent confidential data conditioned
on the proposed parameter, and finally updating the surrogate model based on the
observed noisy summary statistics. As an important application of DP
probabilistic modelling, \cite{RaisaJKH23} proposed a synthetic data generation
method, which is based on learning a noise-aware posterior for a discrete
marginal query based summary statistic.

\textbf{Contributions:}
\begin{compactenum}
    \item We propose a theoretical framework for noise-aware inference and amend the theory of \citet{pmlr-v202-lemos23a} to assess approximate noise-aware posteriors in \cref{formalism-subsection}.
    \item We propose Noise-Aware DP VI (NA-DPVI) for approximate noise-aware inference in \cref{sec:post-processing-model}. The method is based on post-processing the gradient trace from DPVI using a Bayesian linear model to capture the uncertainty from DP, and combining this with data-modelling uncertainty using the VI posterior approximation.
    \item We provide a theoretical analysis of the conditions under which our approach performs well, focusing on how the hyperparameters, including the learning-rate, affect the noise-aware posterior approximation in \Cref{supplemental-section-estimation-of-hessian-learning-rate-relation}.
    \item We employ an accurate evaluation method for approximate noise-aware posteriors by modifying the Test of Accuracy with Random Points (TARP) method in \citep{pmlr-v202-lemos23a}. We apply this algorithm to evaluate our method and compare it against existing baselines in \cref{sec:experiments}. Additionally, we demonstrate the real-world applicability of our method by applying it to a Bayesian logistic regression model on the UCI Adult dataset.
\end{compactenum}

\section{BACKGROUND}
\subsection{Bayesian Inference}
Assume we have a model $p\left(\bfD \mid \btheta\right)$, where $\btheta \in \Theta \subseteq \mathbb{R}^{n}$ denotes the unobserved model parameters; and $\bfD \in \mathcal{D}$ denotes the data. Given a prior $p(\pmb{\theta})$ for $\pmb{\theta}$, Bayes' Theorem states that the posterior can be written as:
\begin{equation}\label{bayes-Theorem}
    p(\pmb{\theta} \mid \bfD) = \frac{p\left(\bfD \mid \btheta\right)p(\btheta)}{\int_{\Theta} p\left(\bfD \mid \btheta\right)p(\btheta) \mathrm{d}\pmb{\theta}}.
\end{equation}
In many cases, the denominator in \cref{bayes-Theorem} is intractable. For such
cases, the true posterior can only be approximated with some other distribution
$\widetilde{p}(\btheta \mid \bfD)$ using methods such as Markov chain Monte
Carlo (MCMC) \citep{neal1993probabilistic} or Variational inference (VI)
\cite{JordanGJS99}.

\subsection{Validating Approximate Bayesian Inference}
After obtaining $\widetilde{p}$, we want to test how well $\widetilde{p}$ approximates $p$. 
In order to devise such a test, let us first assume that $\widetilde{p}$ has support anywhere $p$ has support. 
Next, we define a $(1-\alpha)$ credible region for $\tilde{p}$ as a mapping $\widetilde{\mathcal{R}}_{\alpha}: \mathcal{D} \rightarrow \mathcal{P}(\Theta)$, where $\mathcal{P}$ denotes a power set, such that for any $\bfD \in \mathcal{D}$
\begin{equation}\label{approximate-posterior-coverage}
    \int_{\Theta} \boldone_{\widetilde{\mathcal{R}}_{\alpha}\left(\bfD\right)}\left( \btheta \right) \tildep\left(\btheta \mid \bfD \right) \mathrm{d}\pmb{\theta} 
    = 1 - \alpha.
\end{equation}
Following \cite{pmlr-v202-lemos23a}, we define the Expected Coverage Probability (ECP) of $\widetilde{\mathcal{R}}_{\alpha}\left(\bfD\right)$ as
\begin{align}\label{ecp-definition-original}
        \text{ECP} \left[\widetilde{\mathcal{R}}_{\alpha}\left(\bfD\right)\right] &= \underset{\btheta, \bfD \sim p\left(\btheta, \bfD\right)}{\boldee}\left[\boldone_{\widetilde{\mathcal{R}}_{\alpha}\left(\bfD\right)}\left( \btheta \right) \right].
\end{align}
\cite{pmlr-v202-lemos23a} showed, that if $\text{ECP}
\left[\widetilde{\mathcal{R}}_{\alpha}\left(\bfD\right)\right] = 1-\alpha$ for
every mapping $\widetilde{\mathcal{R}}_{\alpha}$ that satisfied \cref{approximate-posterior-coverage} and every $\alpha \in (0, 1)$,
we have $\tilde{p}(\btheta \mid \bfD) = p(\btheta \mid \bfD)$. Unfortunately, testing
this is not practically feasible, as we would need to enumerate over every
possible credible region. In practice, we need to perform the test over a
finite selection of credible regions, and see if the expected coverage over this
set is close to $1-\alpha$.
However, the way the set of credible regions is constructed is critical as shown by \cite{pmlr-v202-lemos23a}. 
For example, we can obtain perfect $1-\alpha$ coverage by choosing $\widetilde{\mathcal{R}}_{\alpha}\left(\bfD\right)$ to be the ($1 - \alpha$) Highest Posterior Density (HPD) region and then setting $\tildep \left(\btheta \mid \bfD \right) = p(\btheta)$, i.e.\ using the prior as the posterior approximation. 
To solve this problem, \cite{pmlr-v202-lemos23a} introduced the concept of a \textbf{positionable credible region}.
\begin{definition}[\cite{pmlr-v202-lemos23a}] A positionable credible region is a mapping
\begin{equation}
    \tildeCoverage: \mathcal{D} \times \Theta \rightarrow \mathcal{P}\left(\Theta\right),
\end{equation}
for which the following conditions hold for all $\bthetaRef \in \Theta$:
\begin{compactenum}
    \item\label{family-of-credible-regions-positioned-at-ref-1} for any $\alpha \in \left(0, 1\right)$ and $\bfD \in \mathcal{D}$, the set $\tildeCoverage\left(\bfD, \bthetaRef\right)$ is a ($1 - \alpha$) credible region for $\tildep\left(\btheta \mid \bfD\right)$
    \item for all $\bfD \in \mathcal{D}$, \label{family-of-credible-regions-positioned-at-ref-2} $\lim_{\alpha \rightarrow 1}\tildeCoverage\left(\bfD, \bthetaRef\right) = \left\{ \bthetaRef \right\}$.
\end{compactenum}
\end{definition}

Intuitively, this means that we can position credible regions around any point in $\Theta$. In addition, we want to choose $\bthetaRef$ as a function of $\bfD$; that is, $\bthetaRef: \mathcal{D} \rightarrow \Theta$. 
Based on Theorem 3 in \cite{pmlr-v202-lemos23a}, if
\begin{equation} \text{ECP}\left[\tildeCoverage\left(\bfD, \bthetaRef\left(\bfD\right)\right)\right] = 1 - \alpha,
\end{equation}
for all $\alpha \in (0, 1)$ and any function $\bthetaRef: \mathcal{D} \rightarrow \Theta$, then $\tildep\left(\btheta \mid \bfD\right) = p\left(\btheta \mid \bfD\right)$ for all $\btheta \in \Theta$ and $\bfD \in \mathcal{D}$. This is an important property of the ECP with the positionable-credible regions, as it establishes the equality of $\widetilde{p}$ and $p$ if and only if $\widetilde{p}$ has perfect coverages. For the rest of the paper we will abbreviate $\tildeCoverage\left(\bfD, \bthetaRef\left(\bfD\right)\right)$ as $\tildeCoverage\left(\bfD\right)$.

The resulting coverage test by \cite{pmlr-v202-lemos23a} called TARP can
be implemented with the following steps.
\begin{algorithm}
    \begin{algorithmic}[1]
        \State $S_\alpha \gets 0$
        \For {$k \leq K$}
            \State Sample $\btheta, \bfD$ from the model $p(\btheta, \bfD)$;
            \State Compute the approximate posterior $\tilde{p}(\btheta \mid \bfD)$;
            \State Sample reference point $\btheta_{\text{ref}}^{(k)}$
            \State Set $I_k = \boldone_{\widetilde{\mathcal{C}}_\alpha(\bfD)}(\btheta)$;
            \State $S_\alpha = S_\alpha + I_k$
        \EndFor
        \State \textbf{return} $\frac{1}{K} S_{\alpha}$
    \end{algorithmic}
    \caption{TARP method. Modified from \cite{pmlr-v202-lemos23a}.\label{alg:non-private-ecp}}
\end{algorithm}

\subsection{Differential Privacy (DP) and DP Stochastic Gradient Descent}
DP \citep{DworkMNS06} is the standard definition for privacy in modern computer science. DP is defined over so called neighbouring data sets, i.e.~data sets that differ in only single element. The formal definition is given as:
\begin{definition}{\cite{DworkMNS06}}
    A randomized algorithm $\mathcal{A}$ is said to be $(\epsilon, \delta)$-DP if for all neighbouring data sets $D, D'$ and all sets $S \subset \mathrm{Range}(\mathcal{A})$
    \begin{align}
        \Pr(\mathcal{A}(D) \in S) 
            \leq e^\epsilon \Pr(\mathcal{A}(D') \in S) + \delta.
    \end{align}
\end{definition}

DP has a number of appealing theoretical properties, including compositionality, whereby repeated accesses to the data weaken the guarantees in a predictable manner, and post-processing immunity, whereby the guarantees can never be weakened by post-processing.

DP Stochastic Gradient Descent (DP-SGD) \citep{pmlr-v22-rajkumar12, 6736861, Abadi_2016} modifies the traditional SGD algorithm to satisfy DP.
In order to apply DP-SGD, we need to assume that the loss function, parametrized
by $\bphi$, can be represented as a sum over individuals $\pmb{x}_i \in \bfD$:
\begin{equation}\label{dp-sgd-loss-definition}
    \mathcal{L}\left(\bphi ; \bfD \right) = \sum_{i= 1}^{N} \ell(\bphi; \pmb{x}_i).
\end{equation}
DP-SGD makes three modifications to the SGD algorithm: 
\begin{inparaenum}[1)]
    \item gradients are computed for each $\pmb{x}_i$ in the batch
    \item these \emph{per-example} gradients are clipped to a bounded norm $C$ and
    \item Gaussian noise is added to the summed per-example gradients.
\end{inparaenum}
The DP-SGD update rule can be written as
\begin{align}\label{dp-sgd-definition}
    \begin{split}
        \pmb{g}_{t + 1} &= \sum_{i \in \mathcal{B}_{t + 1}} \clip{\nabla_{\pmb{\phi}}\ell(\pmb{\phi}_t; \pmb{x}_i), C}, \\
        \pmb{\phi}_{t + 1} &= \pmb{\phi}_{t} - \lambda_t \left[ \pmb{g}_{t + 1}  + \sigma_{\text{DP}} C \pmb{\eta}_{t + 1} \right],
    \end{split}
\end{align}
where $\lambda_t$ is the learning rate, $\mathcal{B}_t$ is a Poisson subsampled minibatch and $\pmb{\eta}_{t+1} \sim \mathcal{N}(\mathbf{0}, \mathbf{I}_d)$.
The Gaussian noise-level $\sigma_{\text{DP}}$ is chosen so that the $(\epsilon, \delta)$-DP guarantees hold for $T$ iterations of the DP-SGD algorithm. The general privacy proofs for DP-SGD allow releasing all of the intermediate steps and the noisy gradients of the algorithm under the same privacy guarantee.

\subsection{DP Variational Inference}
Variational inference \citep{JordanGJS99} is a method used to approximate a posterior $p\left(\btheta \mid \bfD \right)$ using another distribution called the variational distribution $q_{\text{VI}}\left(\btheta; \bphi\right)$ parameterized by a vector $\bphi \in \Phi \subseteq \mathbb{R}^{d}$. The idea is to find $\pmb{\phi}$ that maximizes the Evidence Lower Bound (ELBO): 
\begin{align}\label{elbo-definition}
    \mathcal{L}\left(\bphi; \bfD\right) 
        &= \boldee_{\qVI(\btheta; \bphi)} \left[ \log \frac{p(\btheta, \bfD)}{q_{\text{VI}}(\btheta; \bphi)} \right].
\end{align}
Furthermore, when the observations are i.i.d.~we can decompose the ELBO as
\begin{align}
    \begin{split}
    \mathcal{L}\left(\bphi; \bfD\right) 
        &= \mathbb{E}_{\qVI(\btheta; \bphi)} \left[ 
                \log \frac{p(\btheta)}{\qVI(\btheta; \bphi)}
            \right] \\
        &\phantom{aaa} + \sum_{i=1}^N \mathbb{E}_{\qVI(\btheta; \bphi)} \left[ 
                \log p(\pmb{x}_i \mid \btheta)
            \right].
    \end{split}
\end{align}
The expectations in ELBO are typically intractable. To solve this,
\cite{KingmaW13} proposed a solution called Stochastic Gradient Variational
Bayes (SGVB), which approximates the expectations with Monte-Carlo estimates and
parametrizes $\qVI$ s.t.~the Monte-Carlo approximator of the ELBO
becomes a differentiable function that can be optimized with gradient based
methods. For more details on the Monte-Carlo estimation and the differentiable
parametrization of $\qVI$ see \cref{supplemental-section-experimental-additional-details}.

\cite{dpviinf} proposed an DP variant of the SGVB method called DPVI, by
replacing the gradient optimizer with DP-SGD. Later \cite{dpvim} improved the 
DPVI by using certain structural knowledge of the gradients in order to 
improve the convergence.

\section{NOISE-AWARE INFERENCE}\label{sec:noise-aware-inference}
\subsection{Formalism}\label{formalism-subsection}
The distributions in this paper are assumed to be continuous unless explicitly stated otherwise. Optimizing the ELBO using DP-SGD introduces uncertainty into the final VI approximation. Our goal is to integrate the noise from the DP mechanism into the posterior, this motivates the following definition.
\begin{definition}\label{definition-noise-aware-posterior}
  Given any prior distribution $p(\btheta)$ for $\btheta$, let $\mathcal{M}\left( \bfD; \epsilon, \delta\right)$ be randomized algorithm that outputs a random vector $\bxi \in \Xi \subseteq \mathbb{R}^{w}$ according to some distribution $p\left(\bxi \mid \bfD, \btheta\right)$ and provides $(\epsilon, \delta)$ differential privacy, then the following posterior distribution given according to Bayes' Theorem,
  \begin{equation}\label{definition-noise-aware-posterior-bayes}
      p\left(\btheta \mid \bxi \right) = \frac{\int_{\calD} p\left( \bxi \mid \bfD, \btheta \right) p\left(\bfD \mid \btheta \right)p(\btheta) \mathrm{d}\mathbf{D}}{\int_{\Theta} \int_{\calD} p\left( \bxi \mid \bfD, \btheta \right) p\left(\bfD \mid \btheta \right)p(\btheta) \mathrm{d}\mathbf{D} \mathrm{d}\pmb{\theta}},
  \end{equation}
  is called the (true) \textbf{noise-aware posterior}.
  The following joint distribution which is found in \cref{definition-noise-aware-posterior-bayes} is very important and will be used repeatedly throughout this paper,
    \begin{equation}\label{noise-aware-joint}
        p\left(\btheta, \bfD, \bxi\right) =  p\left( \bxi \mid \bfD, \btheta \right) p\left(\bfD \mid \btheta \right)p(\btheta).
    \end{equation}
\end{definition}

Similar to $\bfD$, the vector $\bxi$ can be thought of as either one point or a list of points concatenated into one large vector. For instance, DP-SGD provides the iterates $(\bphi_0, \ldots, \bphi_T)$ as a list of vectors. \cref{definition-noise-aware-posterior-bayes} is obtained by applying Bayes' theorem using the joint distribution in \cref{noise-aware-joint} with additionally marginalizing out the data $\mathbf{D}$, as it cannot be observed in the DP setting.

Computing \cref{definition-noise-aware-posterior-bayes} analytically is often intractable, and instead it is approximated by another distribution. Given any approximation of \cref{definition-noise-aware-posterior-bayes}, say $\tildep\left(\btheta \mid \bxi \right)$, we want to evaluate this approximation using the ECP. However, compared to the non-private Bayesian inference, we now have three random vectors $\pmb{\theta}$, $\bfD$, and $\bxi$ in the data generating process. 
Thus, we propose a slight modification to the $\text{ECP}$ as follows. Note that $p\left(\btheta , \bxi\right) = p\left(\bxi \mid \btheta \right) p\left(\btheta\right)$, so we can define the ECP equivalently using the joint distribution in \cref{noise-aware-joint} as,
\begin{align}\label{ecp-definition-alternate}
    \begin{split}
        \text{ECP} \left[\tildeCoverage\left(\bxi\right)\right] &= \underset{\btheta, \bfD, \bxi \sim p\left(\btheta, \bfD, \bxi\right)}{\boldee}\left[\boldone_{\tildeCoverage\left(\bxi\right)}\left( \btheta \right) \right]. 
    \end{split}
\end{align}
The reason why \cref{ecp-definition-alternate} and \cref{ecp-definition-original} are equivalent is that the indicator function $\boldone_{\tildeCoverage\left(\bxi\right)}\left( \btheta \right)$ does not depend on $\bfD$, so by applying Fubini's Theorem \citep{folland2013real}, the expectation in \cref{ecp-definition-alternate} becomes
\begin{align}
    \underset{\btheta, \bfD, \bxi \sim p\left(\btheta, \bfD, \bxi\right)}{\boldee}\left[\boldone_{\tildeCoverage\left(\bxi\right)}\left( \btheta \right) \right]
    = \underset{\btheta, \bxi \sim p\left(\btheta, \bxi\right)}{\boldee}  \left[\boldone_{\tildeCoverage\left(\bxi\right)}(\btheta) \right].
\end{align}
Similarly, the ECP in \cref{ecp-definition-alternate} with the positionable-credible regions establishes the equality of $\widetilde{p}$ and $p$ if and only if $\widetilde{p}$ has perfect coverages. 

To evaluate the approximate noise-aware posterior based on \cref{ecp-definition-alternate}, the coverage test in \cref{alg:non-private-ecp} needs to be modified to include the data generation process for $\bxi$. We first draw samples from the joint distribution \cref{noise-aware-joint}, where the samples are drawn in the following order,
\begin{equation}
\btheta_i \sim p\left(\btheta\right) \boldsymbol{\rightarrow} \bfD_i \sim p\left(\bfD \mid \btheta_{i} \right) \boldsymbol{\rightarrow}  \bxi_i \sim p\left(\bxi_{i} \mid \bfD_{i}, \btheta_{i} \right).
\end{equation}
Second, we replace $I_k$ in~\cref{alg:non-private-ecp} with $I_k = \boldone_{\tildeCoverage\left(\bxi_i\right)}\left(\btheta_i\right)$.

\subsection{Approximate Inference with DPVI}
We build our noise-aware approximate Bayesian inference method on top of DPVI.
Recall that DPVI is based on optimizing the ELBO with DP-SGD. The final posterior
approximation returned by DPVI is the $q(\btheta; \pmb{\phi}_T)$, where $\pmb{\phi}_T$
denote the last parameters returned by the DP-SGD optimization. The DPVI treats 
these parameters as the true optima, and hence disregards all the stochasticity
that was introduces by the Gaussian perturbation. Therefore, the DPVI as such is 
completely unaware of the noise.

However, due to the privacy accounting of DP-SGD, we do not need to limit our
considerations to the last iterate. In fact, we can use the full parameter and
noisy gradient traces, respectively $\mathcal{T} = \{ \bphi_{t} \}_{t=1}^{T}$
and  $\widetilde{\mathcal{G}} = \{ \tildebg_{t} \}_{t=1}^T$, for arbitrary
post-processing. 
In the next Section, we will introduce our post-processing model, which allows
us to model the DP-induced noise, and make the approximate inference
noise-aware.

\subsection{Post-processing Model}
Assume for now that our gradients norms are bounded by $C$. Hence
$\text{clip}(\nabla_{\pmb{\phi}} \ell (\pmb{\phi}_t; x_i), C) =
\nabla_{\pmb{\phi}} \ell (\pmb{\phi}_t; x_i)$. From the DP-SGD update 
equations, we have that the noisy gradient at iteration
$t$ is
\begin{align}
    \label{eq:dp-sgd-grad}
    \tildebg_{t + 1} 
        = \sum_{i \in \mathcal{B}_{t + 1}} \nabla_{\pmb{\phi}} \ell (\pmb{\phi}_t; x_i) 
            + \sigma_{\text{DP}} C \boldsymbol{\eta},
\end{align}
\label{sec:post-processing-model}
\begin{figure}[ht]
  \centering
  \includegraphics{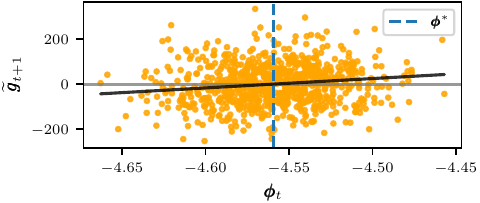}
  \caption{An example of the linear model of the perturbed gradients $\kappa\mathbf{A}(\bphi_t - \bphis)$ (black line) based on \ref{expfam-1}.}
    \label{fig:gradients-fit}
\end{figure}
where $\pmb{\eta} \sim \mathcal{N}\left(0, \mathbf{I}_{d}\right)$ with
$\mathbf{I}_{d}$ denoting a $d$-dimensional identity matrix. Next, we assume that the subsampled gradients could be approximated with a Gaussian distribution through the central limit theorem. This is possible since, in DP-SGD, a large sampling rate or batch size is usually used (e.g., $\kappa = 0.1$). Given the full-data gradient at iteration $t$, $\nabla \mathcal{L}(\pmb{\phi}_t; \mathbf{D})$, we have
\begin{align}
    \label{eq:subsampled-true-grad}
    \sum_{i \in \mathcal{B}_{t + 1}} \nabla_{\pmb{\phi}} \ell (\pmb{\phi}_t; x_i)
        \sim \mathcal{N}\left( 
            \kappa \nabla \mathcal{L}(\pmb{\phi}_t; \mathbf{D}), 
            \Sigma_{\text{sub}}(\bphi_t)
            \right).
\end{align}
Combining \cref{eq:dp-sgd-grad,eq:subsampled-true-grad} yields
\begin{align}
    \tildebg_{t + 1} 
        \sim \mathcal{N}(\kappa \nabla \mathcal{L}(\pmb{\phi}_t; \mathbf{D}), \sigma_{\text{DP}}^2 C^2 \mathbf{I}_d + \Sigma_{\text{sub}}(\bphi_t)).
\end{align}
Since the true gradient, as a data-dependent quantity is clearly
unknown, we need to approximate it. We will approximate the true gradient as a
linear function parameterized with a matrix $\mathbf{A}$ and a vector $\bphis$ through the second-order Taylor approximation of $\mathcal{L}$:
\begin{align}
    \nabla \mathcal{L}(\pmb{\phi}_t; \mathbf{D}) \approx \mathbf{A} (\pmb{\phi}_t - \bphis).
\end{align}
This parametrization is convenient, as it allows us to interpret the $\bphis$ as
the \emph{true optimum} for the VI problem. If we further assume that $\Sigma_{\text{sub}}$ is approximately constant around $\bphis$, then denoting the matrix $\Sigma_{\text{total}} = \sigma_{\text{DP}}^2 C^2 \mathbf{I}_d +
\Sigma_{\text{sub}}$, we can write
\begin{align}\label{gradient-model-trace-approximate}
    \begin{split}
        &\tildebg_{t + 1} \mid \bphi_{t}, \mathbf{A}, \bphis, \Sigma_{\text{sub}} \sim \mathcal{N}\left(\kappa \mathbf{A} \left(\bphi_t - \bphis\right), \Sigma_{\text{total}}\right).
    \end{split}
\end{align}
\Cref{fig:gradients-fit} shows an example of this model based on the exponential distribution \ref{expfam-1}, with the noisy gradients $\tildebg_t$ and a MAP estimate for $\bphis$ together with the fitted gradient. We now have a model for the noisy gradients, and can perform Bayesian inference on
the unknown variables $\mathbf{A}, \bphis$ and $\Sigma_{\text{sub}}$.
Note that the parameters $\mathbf{A}$ and $\bphis$ encapsulate all the information about the latent 
sensitive data $\bfD$. This avoids the costly marginalization of individual samples 
over the latent data, and instead our models scales for arbitrary number of data samples.

After obtaining the posterior for $\mathbf{A}, \bphis$ and $\Sigma_{\text{sub}}$, we can 
form the final noise-aware posterior approximation. 
We start by marginalizing out $\mathbf{A}$ and $\Sigma_{\text{sub}}$ from the post-processing model,
\begin{equation}\label{post-processing-model-marginalized}
    p\left(\bphis \mid \mathcal{T}\right) = \int p\left(\mathbf{A}, \bphis, \Sigma_{\text{sub}} \mid \mathcal{T}\right) \mathrm{d}\mathbf{A} \mathrm{d}\Sigma_{\text{sub}}.
\end{equation}
Then by observing that the joint distribution of $\btheta, \bphis$ conditioned on $\mathcal{T}$ is given by,
\begin{align}\label{post-processing-joint-model}
    \begin{split}
        p(\btheta, \bphis \mid \mathcal{T}) &= p(\btheta \mid \bphis, \mathcal{T}) p(\bphis \mid \mathcal{T}) \\
        &= q_{\text{VI}}(\btheta; \bphis) p(\bphis \mid \mathcal{T}),
    \end{split}
\end{align}
we obtain the final noise-aware approximate posterior $\widetilde{p}\left(\btheta \mid \mathcal{T}\right)$ by marginalizing out $\bphis$ in \cref{post-processing-joint-model},
\begin{equation}\label{approximate-noise-aware-posterior}
    \widetilde{p}(\btheta \mid \mathcal{T}) = \int q_{\text{VI}}(\btheta; \bphis) p(\bphis \mid \mathcal{T}) \mathrm{d}\bphis.
\end{equation}
To provide theoretical justification for our gradient-based model, we first need to state the following assumptions.
\begin{assumption}\label{Assumption-existence-of-optimum-for-objective}
    There exists a stationary point $\bphis$ for $\mathcal{L}$ such that $\nabla_{\bphi} \mathcal{L}\left(\bphis; \bfD\right) = \pmb{0}$.
\end{assumption}
\cref{Assumption-existence-of-optimum-for-objective} is a reasonable assumption for the VI problem to be solvable via DP-SGD.
\begin{assumption}\label{Assumption-taylor-approximation-of-loss}
    We assume the loss function can be well-approximated by the second-order Taylor expansion on a neighborhood around $\bphis$ that includes the stationary part of the trace.
\end{assumption}
Motivation for the second-order Taylor approximation can be found in \cref{supplemental-section-taylor-approximation} which implies that $\mathbf{A}$ is the Hessian. We can rigorously formalize \cref{Assumption-taylor-approximation-of-loss} through the Fr\'{e}chet derivative of $\nabla_{\bphi} \mathcal{L}\left(\bphi; \bfD\right)$ at $\bphis$. For some tolerance $e_{\text{tay}} > 0$, that is problem-dependent, there exists $1 < T^{*} < T$, and an open ball $\mathrm{B}_{r^{*}}(\pmb{\phi}^{*})$ around $\bphis$ with radius $r^{*}$ that contains $\bphi_{t}$ for all $t \geq T^{*}$, such that, for all $\bphi \in \mathrm{B}_{r^{*}}(\pmb{\phi}^{*})$
\begin{equation}\label{frechet-derivative-taylor}
    \frac{\lVert \nabla_{\bphi} \mathcal{L}\left(\bphi; \bfD\right) - \nabla_{\bphi}^2 \mathcal{L}\left(\bphis;\bfD\right)(\bphi - \bphis) \rVert}{\lVert \bphi - \bphis \rVert} < e_{\text{tay}}.
\end{equation}
\begin{assumption}\label{Assumption-per-example-loss-is-lipschitz}
    We assume that the clipping threshold $C$ is high enough such that $C \geq \lVert \nabla_{\bphi} \ell(\bphi_t, \pmb{x}_i)  \rVert$ for all $1 \leq t \leq T$ and for all $1 \leq i \leq N$.
\end{assumption}
Assumption \ref{Assumption-per-example-loss-is-lipschitz} is required to make the clipping operation redundant, allowing us to write the noisy gradient models as in \cref{gradient-model-trace-approximate}.
\begin{assumption}\label{Assumption-bounded-sgd-covariance}
    We assume that for all $1 \leq t \leq T$ the subsampling error is bounded, i.e. there exists $e_{\text{sub}} > 0$ such that
    \begin{equation}
        \underset{\text{Ber}\left(\kappa\right)}{\boldee}\left[ \left\lVert \pmb{g}_{t + 1} - \kappa \nabla_{\pmb{\phi}} \mathcal{L}\left(\pmb{\phi}_t; \mathbf{D}\right) \right\rVert^2 \right] \leq e^2_{\text{sub}}.
    \end{equation}
\end{assumption}
We verify this assumption experimentally through the application of our method. We can approximate $\Sigma_{\text{sub}}$ in principle; however, most of our experiments show that this is not necessary due to the relatively large magnitude of the DP noise, so we can set $\Sigma_{\text{sub}} = \pmb{0}$.
Now we will introduce the following theorem which theoretically justifies our post-processing model.
\begin{theorem}\label{gradient-based-model-Theorem}
    Under Assumptions \ref{Assumption-existence-of-optimum-for-objective}, \ref{Assumption-taylor-approximation-of-loss}, \ref{Assumption-per-example-loss-is-lipschitz}, and \ref{Assumption-bounded-sgd-covariance}, there exists a matrix $\mathbf{A}$ such that
    \begin{equation}\label{gradient-based-model-error}
        \underset{\text{Ber}\left(\kappa\right)}{\boldee} \left[ \left\lVert \pmb{g}_{t + 1} - \kappa \mathbf{A}\left(\pmb{\phi}_t - \bphis \right) \right\rVert^2 \right] < e_{\text{approx}}^2,
    \end{equation}
    where $e_{\text{approx}}^2 = 2e^2_{\text{sub}} + 2(\kappa \times e_{\text{tay}} \times r^{*})^2$.
\end{theorem}
The proof of \cref{gradient-based-model-Theorem} is found in \cref{supplemental-section-proof-of-theorem-1}. 
If some of these assumptions do not hold in practice, such as \cref{Assumption-per-example-loss-is-lipschitz}, then it is sufficient that the approximation given by \cref{gradient-model-trace-approximate} holds for our method to work. These assumptions mainly aid the theoretical analysis and the derivation of our method. We verify that the approximation in \cref{gradient-model-trace-approximate} is reasonable in the following section experimentally, by applying the method to various models. Furthermore, we find that the method is sensitive to the learning rate of DP-SGD and the priors. Thus, we analyze this further in \Cref{supplemental-section-estimation-of-hessian-learning-rate-relation}. We obtain a heuristic for the learning rate based on that analysis, which works well across our experiments.

\begin{figure*}[ht]
  \centering
  \includegraphics[scale=1.0]{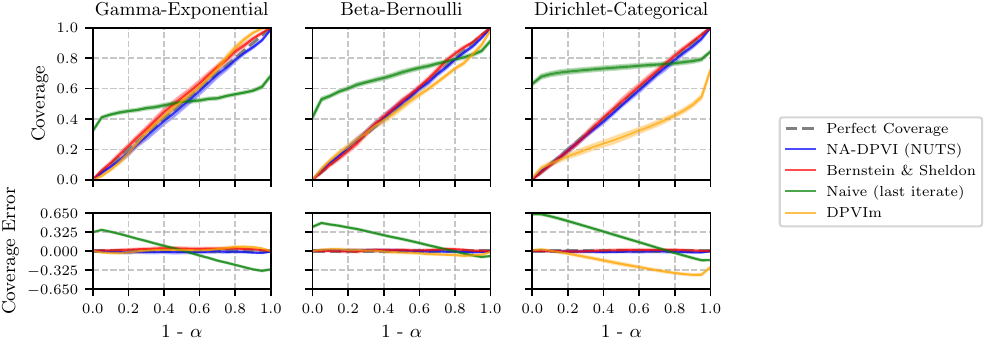}
  \caption{The first row in the figure shows the TARP coverages for the exponential families experiment for NA-DPVI (NUTS), Bernstein \& Sheldon's method \citep{BernsteinS18}, last iterate DPVI, and DPVIm \citep{dpvim}. 
  The second row shows the error for the coverages ($C(\alpha) - (1 - \alpha)$). 
  The solid lines show the average performance over $20$ independent TARP repetitions and the error bars show the corresponding std.
  The parameters for NA-DPVI are $\delta = 10^{-5}$, $N = 5000$, $\kappa = 0.1$ and $T = 10^4$. 
  } 
  \label{fig:expfam-coverages}
\end{figure*}
\label{sec:experiments}
\begin{table*}[ht]
    \centering
    \caption{The RMSE errors corresponding to the exponential family coverages in \cref{fig:expfam-coverages}. Average RMSE $\pm$ std. $\delta = 10^{-5}$ in all experiments.}
    \label{table:expfam-errors}
    \begin{tabular}{lccc}
        \hline
        \textbf{Method} & \textbf{Gamma-Exponential} & \textbf{Beta-Bernoulli} & \textbf{Dirichlet-Categorical}\\
        \hline
        NA-DPVI (NUTS) & \textbf{0.023 $\pm$ 0.008} & \textbf{0.016 $\pm$ 0.006} & 0.020 $\pm$ 0.004\\
        Bernstein \& Sheldon & 0.034 $\pm$ 0.011 & 0.018 $\pm$ 0.006 & \textbf{0.017 $\pm$ 0.007}\\
        Naive (last iterate) & 0.232 $\pm$ 0.003 & 0.273 $\pm$ 0.007 & 0.355 $\pm$ 0.009\\
        DPVIm & 0.038 $\pm$ 0.005 & 0.044 $\pm$ 0.004 & 0.251 $\pm$ 0.011\\
        \hline
    \end{tabular}
\end{table*}

\section{EXPERIMENTS}

\begin{figure*}[ht]
  \centering
  \includegraphics[scale=1.0]{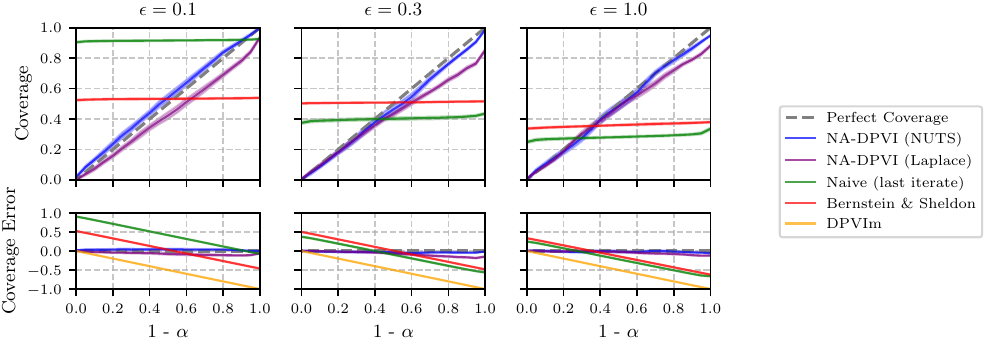}
  \caption{
  The first row in the figure shows the TARP coverages for the 10D Bayesian linear regression experiment for NA-DPVI (NUTS, Laplace), last iterate DPVI, DPVIm \citep{dpvim}, and Gibbs-SS-Noisy \citep{BernsteinS19}. 
  The second row shows the error for the coverages ($C(\alpha) - (1 - \alpha)$). 
  The solid lines show the average performance over $20$ independent TARP repetitions and the error bars show the corresponding std.
  The parameters for NA-DPVI are $\delta = 10^{-5}$, $N = 5000$, $\kappa = 0.1$ and $T = 10^4$. 
  } 
  \label{fig:linear-regression-10d-coverages}
\end{figure*}

\subsection{Setup}\label{sec:experiments-setup}
In our experiments, we use a diagonal Gaussian distribution as the variational distribution, i.e. $q_\text{VI}(\btheta; \bphi) = \mathcal{N}(\btheta; \pmb{\mu}, \mathrm{diag}(\pmb{s}))$, where $\pmb{\mu}$ are the means and $\pmb{s}$ are the variances.
Instead of directly optimizing the variances $\pmb{s} \in \mathbb{R}_{+}^n$, 
we parametrize the variational distribution with $\pmb{u} \in \mathbb{R}^n$ 
such that $\pmb{s}$ is an element-wise softplus transformation of $\pmb{u}$ as
$s_j = \log (1 + e^{u_j})$.
Now the variational distribution is parameterized with $\bphi = (\pmb{\mu},
\pmb{u})$ leading to $d=2n$ variational parameters to optimize in total.

\cite{dpvim} demonstrated that the gradients w.r.t.~$\pmb{u}$ are usually a lot smaller in magnitude than the gradients w.r.t.~$\pmb{\mu}$.
Since DP-SGD adds identically distributed noise to each dimension of the
gradient, the gradients w.r.t.~$\pmb{u}$ will be affected disproportionally by
the noise. To avoid this, we scale up gradients w.r.t.~$\pmb{u}$ before clipping
and noise addition, and revert this as a post-processing before we take the
update step.
For more details about this \emph{preconditioning} see \cref{supplemental-section-experimental-additional-details}.

We infer the post-processing model in \cref{post-processing-joint-model} through two methods, the No-U-Turn Sampler (NUTS) \citep{JMLR:v15:hoffman14a} and using Laplace's approximation which we obtain using the Adam optimizer \citep{DBLP:journals/corr/KingmaB14}. More details about these methods can be found in \cref{supplemental-section-post-processing-model-approximations}. Additionally we compare our method against the last iterate DPVI.

We approximate $\mathbf{A}$ with a diagonal matrix, that is $\mathbf{A} = \mathrm{diag}(a_1, \ldots, a_d)$. Further, from the fact that the Hessian of the loss function $\mathcal{L}$ (here, the negative ELBO minimized by DP-SGD) is positive-definite around its optimum, we know$\mathbf{A}$ to be positive-definite. Therefore, we constrain $\mathbf{A}$ to be element-wise positive by parameterizing the post-processing model with $\pmb{v} \in \mathbb{R}^d$ and mapping the $v_i$ to $a_i$ with softplus transformation, i.e.~$a_i = \text{softplus}(v_i)$. The algorithm blocks for NA-DPVI are provided in \cref{supplemental-section-algorithm-blocks}.

We observed empirically that non-informative priors worked poorly for our
post-processing model which motivated us to do further analysis (see \Cref{supplemental-section-estimation-of-hessian-learning-rate-relation}). Instead, we chose the priors for $\pmb{v}$ and
$\bphis$ based on $\mathcal{T}$ and $\widetilde{\mathcal{G}}$ which we discuss in \cref{supplemental-section-setting-the-priors}.

For evaluating the coverages, we apply a slightly modified version of the TARP algorithm (\cref{alg:non-private-ecp}) according to \cref{formalism-subsection} which we elaborate on further in \cref{supplemental-section-tarp-evaluation-method}. We repeat TARP $20$ times in each experiment with $K=500$ to obtain error estimates for the TARP. We estimate the error as the difference between the coverages at each credible level and the perfect coverages ($C(\alpha) - (1 - \alpha)$). We summarize these coverage errors with a single number using the RMSE (Root Mean Square Error) metric for $N_{\alpha}$ values of $\alpha$,
\begin{equation}\label{rmse-definition}
    \text{RMSE} = \sqrt{\frac{1}{N_{\alpha}} \sum_{i = 1}^{N_{\alpha}} \left(C\left(\alpha_i\right) - \left(1 - \alpha_i\right) \right)^2}.
\end{equation}
Across all the coverage experiments we use $\delta = 10^{-5}$, $T = 10^4$, $N = 5000$, $\kappa = 0.1$, and we use the NUTS method with $1000$ warmup iterations and then run it for $4000$ iterations.

\subsection{Exponential Families}\label{section-expfam}
We compare our approach to the method of \citet{BernsteinS18} which only works for exponential families, and with DPVIm \citep{dpvim}, to show that our method works well for similar models. We  implement the following conjugate models,
\begin{enumerate}[label=M\arabic*,nosep]
    \item\label{expfam-1}: $\pmb{\theta} \sim \text{Gamma}\left(\alpha, \beta\right)$, $\pmb{x} \mid \pmb{\theta} \sim \text{Exp}(\pmb{\theta})$,
    \item\label{expfam-2}: $\pmb{\theta} \sim \text{Beta}\left(\alpha, \beta\right)$, $\pmb{x} \mid \pmb{\theta} \sim \text{Ber}(\pmb{\theta})$,
    \item\label{expfam-3}: $\pmb{\theta} \sim \text{Dir}\left(\alpha_1, \alpha_2, \alpha_3\right)$, $\pmb{x} \mid \pmb{\theta} \sim \text{Cat}(\pmb{\theta})$.
\end{enumerate}
For details on the specific parameter transformations we use for these models
see, \cref{supplemental-section-exponential-families-constrained-optimization}. The coverages for both our method and Bernstein \& Sheldon's are shown in \cref{fig:expfam-coverages} including the coverages for the naive baseline $q_{\text{VI}}(\btheta \mid \bphi_{T})$. We use $\epsilon = 0.1$ for all the models. We also show the errors between the coverages and the perfect coverages in \cref{fig:expfam-coverages}). \cref{table:expfam-errors} summarizes the errors using the RMSE in \cref{rmse-definition}. From the coverages and coverage errors, we can see that both methods work similarly compared to the naive baseline. Our method performs better for \ref{expfam-1} and \ref{expfam-2}; however, Bernstein \& Sheldon's method works better for \ref{expfam-3}.

\begin{figure*}[ht]
  \centering
  \includegraphics[scale=1.0]{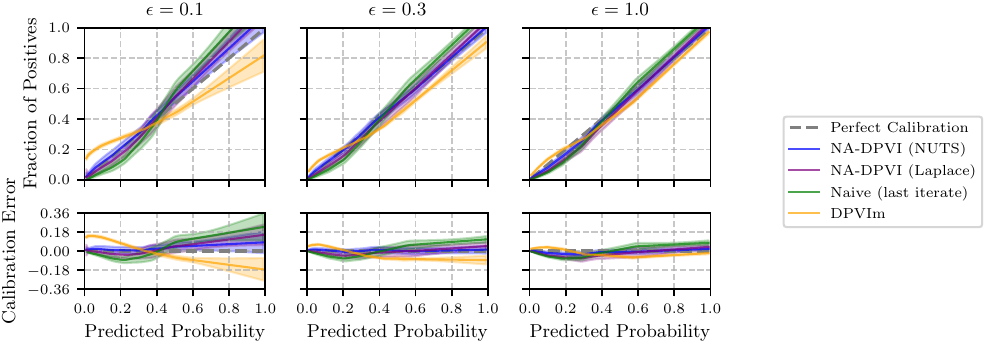}
  \caption{The first row in the figure shows the predictive calibration for the UCI Adult logistic regression experiment for NA-DPVI (NUTS), NA-DPVI (Laplace), last iterate DPVI, and DPVIm \citep{dpvim}. 
  The second row shows the calibration error (Fraction of Positives - Predicted Probability). 
  The solid lines show the average performance over $20$ independent repetitions, and the error bars show the corresponding std.
  The parameters for NA-DPVI are $\delta = 10^{-5}$, $\kappa = 0.1$ and $T = 10^4$.
  } 
  \label{fig:uci-adult-calibrations}
\end{figure*}

\subsection{10D Bayesian linear regression}
Noise-aware Bayesian linear regression was done before by \citep{BernsteinS19}; however their approach does not scale to many dimensions. The baselines we provide are Gibbs-SS-Noisy \citep{BernsteinS19}, DPVIm \citep{dpvim}, and by using the last iterate $\bphi_T$, i.e. constructing credible regions from the distribution $q_{\text{VI}}(\btheta; \bphi_T)$. We apply our method and compare it with the baselines over the following model:
\begin{align}
    \begin{split}
        &\left(\theta_i, \theta_{\sigma}^2\right) \sim \text{N-}\Gamma^{-1}\left(0, \frac{1}{4}, 20, \frac{1}{2}\right), \\
        &x_i \sim \mathcal{N}\left(0, 1\right), \quad y \mid \pmb{x}, \pmb{\theta} \sim \mathcal{N}\left(\pmb{w}^\top \pmb{x}, \theta_{\sigma}^2 \right),
    \end{split}
\end{align}
where the parameters are  $\pmb{\theta} = \left(\theta_1, \ldots, \theta_{10}, \theta_{11}, \theta_{\sigma}^2\right)$ and $\text{N-}\Gamma^{-1}$ is the normal-inverse-gamma distribution. We define $\pmb{w} = (\theta_1, \ldots, \theta_{11})$ and $\pmb{x} = (x_1, \ldots, x_{10}, 1)$. The bias term is $\theta_{11}$. For the constrained optimization problem, we only condition $\theta_{\sigma}^2$ and use the same function as \cref{softplus-conditioning}. 

We compute the coverages for our method using both NUTS and Laplace's approximation for $\epsilon \in \{0.1, 0.3, 1.0 \}$. The coverages for our method is shown in \cref{fig:linear-regression-10d-coverages} including the coverages for the naive baseline $q_{\text{VI}}(\btheta \mid \bphi_{T})$. We also show the errors between the coverages and the perfect coverages in \ref{fig:linear-regression-10d-coverages}. \cref{table:linear-regression-10d} summarizes the errors using the RMSE \cref{rmse-definition}. From the coverages and coverage errors, we can see that both methods work similarly compared to the naive baseline. However, NUTS performs better compared to Laplace's approximation.
\begin{table}[ht]
    \centering
    \caption{The RMSE errors corresponding to the Bayesian linear regression coverages in \cref{fig:linear-regression-10d-coverages}. Average RMSE $\pm$ std, with all the values scaled by $10^3$. $\delta = 10^{-5}$ in all experiments.}
    \resizebox{\columnwidth}{!}{
    \label{table:linear-regression-10d}
    \begin{tabular}{lccc}
        \hline
        \textbf{Method} & \textbf{$\epsilon = 0.1$} & \textbf{$\epsilon = 0.3$} & \textbf{$\epsilon = 1.0$}\\
        \hline
        NA-DPVI (NUTS) & \textbf{36 $\pm$ 11} & \textbf{35 $\pm$ 11} & \textbf{27 $\pm$ 9}\\
        NA-DPVI (Laplace) & 78 $\pm$ 11 & 98 $\pm$ 6 & 60 $\pm$ 7\\
        Naive (last iterate) & 512 $\pm$ 3 & 307 $\pm$ 3 & 360 $\pm$ 3\\
        Bernstein \& Sheldon & 301 $\pm$ 1 & 299 $\pm$ 1 & 323 $\pm$ 3\\
        DPVIm & 584 $\pm$ 0 & 584 $\pm$ 0 & 584 $\pm$ 0\\
        \hline
    \end{tabular}
    }
\end{table}

\subsection{UCI Adult Bayesian Logistic Regression}
We perform an experiment on the UCI Adult dataset \citep{adult_2} with the Bayesian logistic regression model with standard normal prior. No baseline method can scale to this problem other than the naive (last iterate) baseline and DPVIm \citep{dpvim}.
After learning the posterior for $\btheta$, we test the posterior predictive distribution 
given as
\begin{equation}\label{noise-aware-posterior-predictive}
    \widetilde{p}(y = 1 \mid \mathcal{T}, \pmb{x}) =  \int p(y = 1 \mid \pmb{\theta}, \pmb{x}) \widetilde{p}(\pmb{\theta}\mid \mathcal{T}) \mathrm{d}\btheta,
\end{equation}
where we replace the integral with its Monte-Carlo estimator.
We do $20$ repeats for each experiment and we plot the calibration curves for NA-DPVI (NUTS \& Laplace) and for the other baselines with $\epsilon \in \{0.1, 0.3, 1.0 \}$. The calibration curves and calibration errors are found in \cref{fig:uci-adult-calibrations}. We also compute the RMSE for the calibration errors in \cref{table:uci-adult-rmse}, this is the same as the square root of the Brier score \citep{VERIFICATIONOFFORECASTSEXPRESSEDINTERMSOFPROBABILITY}. From the calibrations and calibration errors, we can see that both NUTS and Laplace outperform the naive baseline and DPVIm. However, NUTS performs slightly better than Laplace's approximation. For more details about the experiment, see \Cref{uci-adult-experiment-details}.
\begin{table}[ht]
    \centering
    \caption{The RMSE errors corresponding to the logistic regression calibration in \cref{fig:uci-adult-calibrations}. Average RMSE $\pm$ std, with all the values scaled by $10^3$. $\delta = 10^{-5}$ in all experiments.}
    \label{table:uci-adult-rmse}
    \resizebox{\columnwidth}{!}{
    \begin{tabular}{lccc}
        \hline
        \textbf{Method} & \textbf{$\epsilon = 0.1$} & \textbf{$\epsilon = 0.3$} & \textbf{$\epsilon = 1.0$}\\
        \hline
        NA-DPVI (NUTS) & \textbf{61 $\pm$ 31} & \textbf{26 $\pm$ 5} & \textbf{24 $\pm$ 7}\\
        NA-DPVI (Laplace) & 84 $\pm$ 43 & 44 $\pm$ 11 & 46 $\pm$ 14\\
        Naive (last iterate) & 120 $\pm$ 65 & 67 $\pm$ 24 & 54 $\pm$ 17\\
        DPVIm & 101 $\pm$ 52 & 65 $\pm$ 25 & 38 $\pm$ 14\\
        \hline
    \end{tabular}
    }
\end{table}

\section{DISCUSSION}

Bayesian inference should be a very natural companion to DP because it naturally deals with uncertain information and should thus be able to easily tolerate the noise added for DP.
The fact that noise-aware inference is so difficult reminds us that commonly used general-purpose inference methods such as MCMC are not really Bayesian since they do not incorporate a mechanism for reasoning about the accuracy of their results.
In this sense, algorithms required for noise-aware inference share a similarity with probabilistic numerics \citep{HennigOK2022}.

We cite DPVIm \citep{dpvim} in our work regarding DPVI, which they extensively investigate; however, the noise-awareness concept in their work is not exactly equivalent to ours which we generalized from \citet{BernsteinS19}. Although our proposed NA-DPVI greatly expands the domain of noise-aware inference, it is still restricted by its reliance on potentially inaccurate VI to capture the data-modelling uncertainty. Additionally, our approximations only work under certain conditions, which may not always hold. An explicit privacy-utility-computation trade-off is very difficult to obtain for NA-DPVI, because there are several factors that are challenging to account for, theoretically speaking. These factors include DP-SGD, VI, the post-processing model, and the final approximate posterior obtained either in closed form or through an approximate inference method (MCMC, Laplace).

Another limitation in this work is that we did not account for the privacy leakage from hyper-parameter selection, which was done manually due to the high computational cost. Currently, many papers that employ DP-SGD ignore the privacy accounting of hyper-parameter tuning. For example, see Section 5.2 of \citet{Sander2023}. We applied heuristics to choose some of the hyper-parameters (e.g., the learning rate), and other hyper-parameters were shared among most experiments which reduces the privacy leakage. The state of privacy accounting for hyper-parameter tuning is still at its infancy, see Section 5.4 of \citet{Ponomareva2023}, and this presents an area for future research and improvement, especially for noise-aware inference.

The search for the ultimate noise-aware inference algorithm that could accurately capture both data-modelling uncertainty as well as uncertainty due to DP inference for arbitrary models remains an important goal for future research.

\subsubsection*{Code availability}
The implementation of NA-DPVI and the code for reproducing our experiments are publicly available at \url{https://github.com/DPBayes/noise-aware-dpvi}.

\acknowledgments\label{sec:acknowledgments}

This work was supported by the Research Council of Finland (Flagship programme: Finnish Center for Artificial Intelligence, FCAI, Grant 356499 and Grant 359111), the Strategic Research Council at the Research Council of Finland (Grant 358247) as well as the European Union (Project 101070617).
Views and opinions expressed are however those of the author(s) only and do not necessarily reflect those of the European Union or the European Commission. 
Neither the European Union nor the granting authority can be held responsible for them. 
This work has been performed using resources provided by the CSC – IT Center for Science, Finland (Project 2003275).

\bibliographystyle{myabbrvnat}
\bibliography{references}

\appendix

\setcounter{figure}{0}
\renewcommand\thefigure{A.\arabic{figure}}
\setcounter{table}{0}
\renewcommand{\thetable}{A\arabic{table}}
\setcounter{equation}{0}
\renewcommand{\theequation}{A\arabic{equation}}

\section*{Checklist}
 \begin{enumerate}

 \item For all models and algorithms presented, check if you include:
 \begin{enumerate}
   \item A clear description of the mathematical setting, assumptions, algorithm, and/or model. [Yes] Refer to \cref{sec:noise-aware-inference}, \cref{sec:experiments-setup}, and \cref{supplemental-section-algorithm-blocks}.
   \item An analysis of the properties and complexity (time, space, sample size) of any algorithm. [Yes] Refer to \cref{supplemental-section-algorithm-blocks} for the time complexity of NA-DPVI.
   \item (Optional) Anonymized source code, with specification of all dependencies, including external libraries. [Yes]
 \end{enumerate}

 \item For any theoretical claim, check if you include:
 \begin{enumerate}
   \item Statements of the full set of assumptions of all theoretical results. [Yes] Refer to \cref{sec:noise-aware-inference} and \cref{supplemental-section-estimation-of-hessian-learning-rate-relation}.
   \item Complete proofs of all theoretical results. [Yes] Refer to \cref{supplemental-section-estimation-of-hessian}, \cref{supplemental-section-relation-learning-rate-taylor}, and \cref{supplemental-section-heuristic-learning-rate}.
   \item Clear explanations of any assumptions. [Yes] Refer to \cref{sec:noise-aware-inference} and \cref{supplemental-section-estimation-of-hessian-learning-rate-relation}.  
 \end{enumerate}

 \item For all figures and tables that present empirical results, check if you include:
 \begin{enumerate}
   \item The code, data, and instructions needed to reproduce the main experimental results (either in the supplemental material or as a URL). [Yes/No/Not Applicable]
   \item All the training details (e.g., data splits, hyperparameters, how they were chosen). [Yes] Refer to \cref{sec:experiments} and \cref{supplemental-section-experimental-additional-details}.
         \item A clear definition of the specific measure or statistics and error bars (e.g., with respect to the random seed after running experiments multiple times). [Yes] Refer to \cref{sec:experiments} and \cref{supplemental-section-experimental-additional-details}.
         \item A description of the computing infrastructure used. (e.g., type of GPUs, internal cluster, or cloud provider). [Yes] Refer to \cref{sec:acknowledgments} and \cref{supplemental-computational-resources}.
 \end{enumerate}

 \item If you are using existing assets (e.g., code, data, models) or curating/releasing new assets, check if you include:
 \begin{enumerate}
   \item Citations of the creator If your work uses existing assets. [Yes]
   \item The license information of the assets, if applicable. [Not Applicable]
   \item New assets either in the supplemental material or as a URL, if applicable. [Not Applicable]
   \item Information about consent from data providers/curators. [Not Applicable]
   \item Discussion of sensible content if applicable, e.g., personally identifiable information or offensive content. [Not Applicable]
 \end{enumerate}

 \item If you used crowdsourcing or conducted research with human subjects, check if you include:
 \begin{enumerate}
   \item The full text of instructions given to participants and screenshots. [Not Applicable]
   \item Descriptions of potential participant risks, with links to Institutional Review Board (IRB) approvals if applicable. [Not Applicable]
   \item The estimated hourly wage paid to participants and the total amount spent on participant compensation. [Not Applicable]
 \end{enumerate}

 \end{enumerate}

\onecolumn

\section{Second-Order Taylor Approximation of the Loss Function}\label{supplemental-section-taylor-approximation}
For any function $\mathcal{L}(\bphi)$, if $\nabla \mathcal{L}\left(\bphis\right) = 0$, then by using the second-order Taylor approximation around $\bphis$, we can write $\mathcal{L}(\bphi)$ as:
\begin{align}\label{second-order-taylor-approximation}
    \begin{split}
        \mathcal{L} \left(\pmb{\phi} \right) &\approx \mathcal{L} \left(\pmb{\phi}^{*} \right) + \nabla\mathcal{L} \left(\pmb{\phi}^{*} \right)^T (\pmb{\phi} - \pmb{\phi}^*) + \frac{1}{2}(\pmb{\phi} - \pmb{\phi}^*)^T\nabla^2 \mathcal{L} \left(\pmb{\phi}^{*} \right)(\pmb{\phi} - \pmb{\phi}^*) \\
        &\approx \mathcal{L} \left(\pmb{\phi}^{*} \right) + \frac{1}{2}(\pmb{\phi} - \pmb{\phi}^*)^T\nabla^2 \mathcal{L} \left(\pmb{\phi}^{*} \right)(\pmb{\phi} - \pmb{\phi}^*),
    \end{split}
\end{align}
where $\nabla^2 \mathcal{L} \left(\pmb{\phi}^{*} \right)$ is the hessian of $\mathcal{L}$ at $\pmb{\phi}^{*}$. Taking the gradient of both sides of (\ref{second-order-taylor-approximation}) with respect to $\pmb{\phi}$, we obtain,
\begin{equation}
    \nabla\mathcal{L}\left(\pmb{\phi} \right) \approx \nabla^2\mathcal{L}\left(\pmb{\phi}^{*} \right) \left(\pmb{\phi} - \pmb{\phi}^*\right).
\end{equation}

\section{Proof of Theorem 1}\label{supplemental-section-proof-of-theorem-1}
\begin{proof}
    Let $\mathbf{A} = \nabla_{\bphi}^2 \mathcal{L}\left(\bphis; \bfD\right)$, then for all $t \geq T^{*}$, 
\begin{align*}
    &\left\lVert \pmb{g}_{t + 1} - \kappa  \nabla_{\pmb{\phi}} \mathcal{L}\left(\pmb{\phi}_t; \mathbf{D}\right) + \kappa \nabla_{\pmb{\phi}} \mathcal{L}\left(\pmb{\phi}_t; \mathbf{D}\right)  - \kappa  \mathbf{A}\left(\pmb{\phi}_t - \bphis \right) \right\rVert^2 \\
    &\leq 2\left\lVert \pmb{g}_{t + 1} - \kappa  \nabla_{\pmb{\phi}} \mathcal{L}\left(\pmb{\phi}_t; \mathbf{D}\right)\right\rVert^2  +  2\kappa^2 \left\lVert \nabla_{\pmb{\phi}} \mathcal{L}\left(\pmb{\phi}_t; \mathbf{D}\right)  - \mathbf{A}\left(\pmb{\phi}_t - \bphis \right) \right\rVert^2 \\
    &\leq 2\left\lVert \pmb{g}_{t + 1} -  \kappa \nabla_{\pmb{\phi}} \mathcal{L}\left(\pmb{\phi}_t; \mathbf{D}\right)\right\rVert^2 + 2\kappa^2 e_{\text{tay}}^2 \lVert\pmb{\phi}_t - \bphis \rVert^2 ,
\end{align*}
    taking the expectation of both sides with respect to the subsampling distribution and using the fact that for all $t \geq T^{*}$, $\lVert\pmb{\phi}_t - \bphis \rVert^2 < (r^{*})^2$, we obtain
    \begin{equation}
    \underset{\text{Ber}\left(\kappa\right)}{\boldee} \left[ \left\lVert \pmb{g}_{t + 1} - \kappa \mathbf{A}\left(\pmb{\phi}_t - \bphis \right) \right\rVert^2 \right] < e_{\text{approx}}^2.
    \end{equation}
    where $\mathrm{Ber}(\kappa)$ is the Bernoulli distribution with $p = \kappa$, which is the subsampling distribution.
\end{proof}

\section{Noise-Aware Posterior Approximation}\label{supplemental-section-post-processing-model-approximations}
Deriving the closed form of $p(\bphis, \mathbf{A} \mid \mathbf{T})$ might not always be feasible so we obtain the following approximations. The first Laplace's approximation. Since the Hessian is positive-definite, we want $\mathbf{A}$ to be positive-definite as well, so we need a bijective differentiable function (diffeomorphism) $F$, such that $F(\mathbf{V}) = \mathbf{A}$ where $\mathbf{V}$ is an unconstrained version of $\mathbf{A}$. Sometimes, $F$ is referred to as a conditioning transformation \citep{Blei_2017}. Therefore,
\begin{equation*}
    p\left(\mathbf{V},\bphis \mid \mathcal{T} \right) = p\left(F(\mathbf{V}),\bphis \mid \mathcal{T} \right) \left\lvert J_{F}\left(\mathbf{V}\right) \right\rvert,
\end{equation*}
where $J_{F}$ is the Jacobian of $F$. To obtain Laplace's approximation first we need the MAP (Maximum a posteriori) estimate for both $\mathbf{V}$ and $\bphis$,
\begin{equation}\label{laplace-map-estimates}
    \left(\widehat{\mathbf{V}}_{\text{map}}, \widehat{\bphi}^{*}_{\text{map}}\right) = \underset{\left(\bphis, \mathbf{V}\right)}{\text{argmax}}\; \log p\left(\mathbf{V},\bphis \mid \mathcal{T} \right).
\end{equation}
Let $p_{\text{gauss}}\left(\mathbf{V}, \bphis\right)$ be the distribution
\begin{equation}\label{laplace-gaussian-distribution}
    \mathcal{N}\left( \left( \text{vec}\left(\widehat{\mathbf{V}}_{\text{map}}\right), \widehat{\bphi}^{*}_{\text{map}}\right), \Sigma_{\text{lap}} \right),
\end{equation}
where $\mathrm{vec}(\cdot)$ is the vector representation of a matrix. 
We can easily marginalize out $\mathbf{V}$ from $p_{\text{gauss}}\left(\mathbf{V}, \bphis\right)$ to obtain the marginal distribution of $\bphis$ and approximate (\ref{post-processing-model-marginalized}) since it is a Guassian distribution. We then use this marginalized distribution to approximate (\ref{approximate-noise-aware-posterior}) and obtain a new distribution that we will denote $\widetilde{p}_{\text{lap}}\left(\btheta \mid \mathcal{T} \right)$.
The second method to approximate (\ref{approximate-noise-aware-posterior}) is by using any Markov Chain Monte Carlo (MCMC) method to approximately sample from the post-processing model,
\begin{equation}\label{mcmc-samples}
    \left(\mathbf{A}_{i}, \bphis_{i}\right) \sim p\left( \mathbf{A}, \bphis \mid \mathcal{T} \right),\; 1 \leq i \leq N_{\text{mc}}.
\end{equation}
We can use these samples to approximate (\ref{approximate-noise-aware-posterior}) using the following mixture model,
\begin{equation}\label{mcmc-mixture-model}
        \widetilde{p}(\btheta \mid \mathcal{T}) \approx \frac{1}{N_{\text{mc}}} \sum_{i = 1}^{N_{\text{mc}}} q_{\text{VI}}\left(\btheta; \bphis_i \right) \equiv_{\text{def}} \widetilde{p}_{\text{mc}}(\btheta \mid \mathcal{T}).
\end{equation}

\section{Estimation of the Hessian Matrix and the Influence of the Learning Rate}\label{supplemental-section-estimation-of-hessian-learning-rate-relation}
From our experiments, we noticed that the uncertainty estimates of the approximate noise-aware posterior are sensitive to the estimation of the Hessian matrix $\mathbf{A}$. Through observation, we found that the accuracy of the estimation of $\mathbf{A}$ is highly correlated with the choice of the constant learning rate $\lambda$. This motivated us to conduct further analysis which we present in this section.

There is also an equivalent model to (\ref{gradient-model-trace-approximate}) which we call parameter-based model of the trace that is useful in this analysis,
\begin{equation}\label{parameter-model-trace-approximate}
    \bphi_{t + 1} \mid \bphi_{t}, \mathbf{A}, \bphis \sim \mathcal{N}\left(\bphi_{t} - \lambda \kappa\mathbf{A} \left(\bphi_t - \bphis\right), \lambda^2 \Sigma_{total}\right),
\end{equation}
where we assume that $\Sigma_{\text{sub}} = \pmb{0}$.

For Assumption \ref{Assumption-taylor-approximation-of-loss} to hold, the part of the trace $\mathcal{T}$ that is stationary around the optimum $\bphis$ (after $t = T^{*}$) needs to be close enough to $\bphis$. On the other hand, the trace should have enough variability around $\bphis$ for the inference to be accurate. Intuitively speaking, we want to sample parameters $\bphi_t$ and gradients $\pmb{g}_t$ around $\bphis$ to use that information for the post-processing model; and in particular, to estimate the Hessian that describes the curvature of the loss function around $\bphis$. For a rigorous analysis, we present the following Theorem with the full proof in Supplement \ref{supplemental-section-estimation-of-hessian}.
\begin{theorem}\label{var-of-hessian-trace-trade-off-Theorem}
    Under assumptions \ref{Assumption-existence-of-optimum-for-objective}, \ref{Assumption-taylor-approximation-of-loss}, \ref{Assumption-per-example-loss-is-lipschitz}, and \ref{Assumption-bounded-sgd-covariance} with the additional assumptions that $\mathbf{A}$ is a diagonal matrix, i.e. $\mathbf{A} = \mathrm{diag}(a_1, \ldots, a_d)$, and that $\bphis$ is known, if we set the prior for $\mathbf{A}$ to any uniform distribution $a_i \sim \text{U}(\cdot, \cdot)$, then the MAP estimate of $\mathbf{A}$, $\widehat{\mathbf{A}}_{\text{map}} = \mathrm{diag}(\hat{a}_1, \ldots, \hat{a}_d)$ has to satisfy the following inequality,
    \begin{equation}\label{var-of-hessian-taylor-trade-off-inquality}
        r^{*} \times \underset{{1 \leq i \leq d}}{\max} \sqrt{\text{Var}\left[ \hat{a}_i \right]} \geq \frac{1}{\kappa} \sqrt{\frac{d}{T - T^{*}}}\sigma_{\text{DP}} C.
    \end{equation}
\end{theorem}
The analysis in \Cref{var-of-hessian-trace-trade-off-Theorem} can be done without the assumption that $\mathbf{A}$ is diagonal; however, this assumption makes the analysis easier and the relation more clear. We also perform our experiments with a diagonal $\mathbf{A}$ matrix so it makes sense to analyze this particular case. Inequality (\ref{var-of-hessian-taylor-trade-off-inquality}) entails that there is a trade-off between $\max_i \sqrt{\text{Var}\left[ \hat{a}_i \right]}$ and $r^{*}$ (the radius of the open ball that contains the iterates in $\mathcal{T}^{*}$), and this puts a limit on how well we can estimate the Hessian with a non-informative prior even if we assume perfect knowledge about $\bphis$.

Also, according to (\ref{parameter-model-trace-approximate}), other than $\Sigma_{\text{DP}}$, it is clear that the value of $\lambda$ affects the variance of the trace around $\bphis$. We show the relation between $\lambda$ and $ \boldee \left[ \lVert \bphi_{t} - \bphi^{*} \rVert^2 \right]$ through the Ornstein–Uhlenbeck (OU) SDE which is a known tool for analyzing variants of SGD \citep{DBLP:conf/icml/MandtHB16,DBLP:journals/jmlr/LiTE19,DBLP:conf/nips/LiMA21}. This is usually done by modeling the noise in the discrete-time SGD process through the Wiener process and then approximating it by a continuous-time process using the It\^{o} integral \citep{Rogers_Williams_2000}.
\begin{theorem}\label{stationary-variance-Theorem}
    Under assumptions \ref{Assumption-existence-of-optimum-for-objective}, \ref{Assumption-taylor-approximation-of-loss}, \ref{Assumption-per-example-loss-is-lipschitz}, and \ref{Assumption-bounded-sgd-covariance}, if the DP-SGD process can be well-approximated by the OU SDE, then the stationary mean of $\lVert \bphi_t - \bphis \rVert^2$ is proportional to $\lambda \sigma_{\text{DP}}^2 C^2$. Moreover, for the special case when $\mathbf{A}$ is a non-singular diagonal matrix,
    \begin{equation}
        \lambda < \frac{2 \kappa (r^{*})^2}{\sigma_{\text{DP}}^2 C^2 \mathbf{Tr}\left( \mathbf{A}^{-1} \right) }.
    \end{equation}
\end{theorem}
What \Cref{stationary-variance-Theorem} entails is that
\begin{equation*}
    \lambda \sigma_{\text{DP}}^2 C^2 \propto \boldee \left[ \lVert \bphi_{t} - \bphi^{*} \rVert^2 \right] < (r^{*})^2,
\end{equation*}
so if we make $\lambda$ large enough, then $\boldee \left[ \lVert \bphi_{t} - \bphi^{*} \rVert^2 \right]$ will increase proportional to $\lambda$ which will invalidate \Cref{Assumption-taylor-approximation-of-loss} (the second-order Taylor approximation of the loss function) by contraposition. The full proof of \Cref{stationary-variance-Theorem} and a further discussion is found in \Cref{supplemental-section-relation-learning-rate-taylor}.

To find a good learning rate, we analyze the convergence of DP-SGD with some additional Assumptions and then find the learning rate that optimizes the convergence bound given in the following theorem. Also, we want to write the optimization problem as a minimization problem. If the optimization problem is initially a maximization problem, then we change the direction of optimization by re-defining $\mathcal{L}$ as $-\mathcal{L}$.
\begin{theorem}\label{dpsgd-convergence-bound-theorem}
    Under assumptions \ref{Assumption-per-example-loss-is-lipschitz} and \ref{Assumption-bounded-sgd-covariance}, further assume that $\mathcal{L}$ has a lower bound, namely $M$, and the gradient $\nabla_{\bphi} \mathcal{L}$ is Lipschitz continuous with Lipschitz constant $H$ (so $\mathcal{L}$ is Lipschitz smooth), and that the per-examples losses are Lipschitz continuous, then the following inequality holds,
    \begin{align}\label{equation-convergence-bound-dpsgd}
         \frac{1}{T}\sum_{t = 0}^{T - 1} \boldee \left[ \left\lVert \nabla \mathcal{L}(\bphi_{t}) \right\rVert^2 \right]  \leq \frac{\mathcal{L}\left(\bphi_{0}\right) - M}{T \lambda \kappa} + \frac{H}{2 \kappa} \left( 2e_{\text{sub}}^2 + 2\kappa^2 G^2 + \sigma_{\text{DP}}^2 C^2 d \right)  \lambda .
    \end{align}
\end{theorem}
We obtain a heuristic based on finding $\lambda$ that optimizes the right-hand side of (\ref{equation-convergence-bound-dpsgd}) by taking the derivative, setting it to zero, and then solving for $\lambda$ from which we obtain a quantity proportional to the following,
\[
    \lambda_{\text{heur}} = \frac{\sqrt{2} \lambda_{c}}{ \sigma_{\text{DP}} C \sqrt{T d} },
\]
where $\lambda_{c} > 0$ is a hyper-parameter. See supplement \ref{supplemental-section-heuristic-learning-rate} for the full proof and other details. For our experiments, we mostly use $\lambda_{c} = 1$ or $\lambda_{c} = \sqrt{\frac{d}{2}}$. This heuristic makes it easier to tune the learning-rate because it sets $\lambda$ to the appropriate scale that is required by convergence results. We show experimentally through our experiments that this is a good choice for the learning rate.

\section{Proof of Theorem 2}\label{supplemental-section-estimation-of-hessian}
As said before, an intuitive explanation why the variance of the trace around $\bphis$ affects the estimation of the hessian is that assuming no to little variability, i.e. $\bphi_t \approx \bphis$ for $t = T^{*}, \ldots, T$, then $\mathbf{A}\left(\bphi_t - \bphis \right) \approx \nabla \mathcal{L}\left(\bphi_t; \mathbf{D}\right)$ becomes $\mathbf{A}\left(\bphis - \bphis\right) \approx \nabla\mathcal{L}\left(\bphis; \mathbf{D}\right) = 0$. This implies that $\mathbf{A} \times 0 \approx 0$, and thus, in this case, it is impossible to estimate $\mathbf{A}$ from this equation. The following is the proof of \Cref{var-of-hessian-trace-trade-off-Theorem}.
\begin{proof}
    The MAP estimate of $\mathbf{A}$ under a inform prior and with perfect knowledge about $\bphis$ is the same as the maximum likelihood estimate (MLE). Using the gradient-based model of the trace, and assuming that the subsampling noise $\Sigma_{\text{sub}}$ is $\mathbf{0}$, the likelihood of $\mathbf{A}$ could be written as:
    \[
        L\left(\mathbf{A} ; \tildebg_{T^{*} + 1}, \ldots, \tildebg_{T}, \bphi_{T^{*}}, \ldots, \bphi_{T - 1}\right) \propto \prod_{t = T^{*}}^{T - 1} p(\tildebg_{t + 1} \mid \bphi_t, \mathbf{A}),
    \]
    taking the logarithm of both sides:
    \[
        \log L\left(\mathbf{A}; \tildebg_{T^{*} + 1}, \ldots, \tildebg_{T}, \bphi_{T^{*}}, \ldots, \bphi_{T - 1}\right) \underset{c}{=} \sum_{t = T^{*}}^{T - 1} \log p\left(\tildebg_{t + 1} \mid \bphi_t, \mathbf{A}\right),
    \]
    where $\underset{c}{=}$ means equals with a constant difference. For each $t$, we have:
    \[
        \log p\left(\tildebg_{t + 1} \mid \bphi_t, \mathbf{A}\right) = -\frac{d}{2} \log 2\pi -  \frac{1}{2}\ \log \det\left(\Sigma_{\text{DP}}\right)- \frac{1}{2} \left( \tildebg_{t + 1} -\kappa  \mathbf{A}(\bphi_t - \bphis) \right)^\top\Sigma_{\text{DP}}^{-1}\left( \tildebg_{t + 1} - \kappa \mathbf{A}(\bphi_t - \bphis) \right),
    \]
    therefore,
    \begin{align}\label{hessian-negative-log-likelihood}
        - \log L\left(\mathbf{A} \right) \underset{c}{=} \frac{1}{2} \sum_{t=T^{*}}^{T - 1} \left( \tildebg_{t + 1} -\kappa  \mathbf{A}(\bphi_t - \bphis) \right)^\top\Sigma_{\text{DP}}^{-1}\left( \tildebg_{t + 1} - \kappa \mathbf{A}(\bphi_t - \bphis) \right).
    \end{align}
    
    If we want to find $\widehat{\mathbf{A}}$ that maximizes the likelihood $L(\mathbf{A})$, then we need to minimize $- \log L\left(\mathbf{A} \right)$ in (\ref{hessian-negative-log-likelihood}). Thus, we have to minimize the sum:
    \begin{equation}\label{full-hessian-least-squares}
        \sum_{t=T^{*}}^{T - 1} \left\lVert \tildebg_{t + 1} - \kappa  \mathbf{A}(\bphi_t - \bphis) \right\rVert^2.
    \end{equation}
    Let $\mathbf{a} = (a_1, \ldots, a_d)$, then we can also re-write the sum (\ref{full-hessian-least-squares}) as:
    \begin{equation}
        \sum_{t=T^{*}}^{T - 1} \left\lVert \tildebg_{t + 1} - \kappa  \mathbf{a} \odot (\bphi_t - \bphis) \right\rVert^2,
    \end{equation}
    where the operation $\odot$ is element-wise multiplication. Hence, we want to find $\hat{\mathbf{a}}$ such that for all $1 \leq i \leq d$:
    \[
    \frac{\partial}{\partial \hat{a}_i} \sum_{t=T^{*}}^{T - 1} \left( \tildebg_{t + 1}^{(i)} - \kappa \hat{a}_{i} \left(\bphi_t^{(i)} - \bphi^{*(i)} \right) \right)^2 = 2 \kappa  \sum_{t=T^{*}}^{T - 1} \left(\kappa \hat{a}_i\left(\bphi_t^{(i)} - \bphi^{*(i)}\right) - \tildebg_{t + 1}^{(i)}  \right) \left(\bphi_t^{(i)} - \bphi^{*(i)}\right) = 0,
    \]
    where $\pmb{v}^{(i)}$ denotes the $i$th component of a vector $\pmb{v}$, in other words,
    \[
        \kappa \hat{a}_i \sum_{t=T^{*}}^{T - 1} \left(\bphi_t^{(i)} - \bphi^{*(i)}\right)^2 = \sum_{t=T^{*}}^{T - 1} \tildebg_{t + 1}^{(i)} \left(\bphi_t^{(i)} - \bphi^{*(i)}\right),
    \]
    \begin{equation}\label{hessian-diagonal-mle}
        \hat{a}_{i} = \frac{\sum_{t=T^{*}}^{T - 1} \tildebg_{t + 1}^{(i)} \left(\bphi_t^{(i)} - \bphi^{*(i)}\right)}{\kappa \sum_{t=T^{*}}^{T - 1} \left(\bphi_t^{(i)} - \bphi^{*(i)}\right)^2}.
    \end{equation}
    
    We are interested in the distribution of each $\hat{a}_i \mid \mathcal{T}$. The gradient-based model implies that:
    \begin{align}\label{hessian-estimator-distribution-derivation}
        \begin{split}
            \tildebg_{t + 1}^{(i)} \mid \mathcal{T} &\sim \mathcal{N}\left( \kappa a_i \left(\bphi_{t}^{(i)} - \bphi^{*(i)} \right), \sigma_{\text{DP}}^2 C^2 \right), \\
            \tildebg_{t + 1}^{(i)} \left(\bphi_{t}^{(i)} - \bphi^{*(i)}\right) \mid \mathcal{T} &\sim \mathcal{N}\left(\kappa  a_i \left(\bphi_{t}^{(i)} - \bphi^{*(i)} \right)^2, \sigma_{\text{DP}}^2 C^2 \left(\bphi_{t}^{(i)} - \bphi^{*(i)} \right)^2\right), \\
            \sum_{t = T^{*}}^{T - 1} \tildebg_{t + 1}^{(i)} \left(\bphi_{t}^{(i)} - \bphi^{*(i)}\right) \mid \mathcal{T} &\sim \mathcal{N}\left(\kappa  a_i \sum_{t = T^{*}}^{T - 1} \left(\bphi_{t}^{(i)} - \bphi^{*(i)} \right)^2,\sigma_{\text{DP}}^2 C^2 \sum_{t = T^{*}}^{T - 1} \left(\bphi_{t}^{(i)} - \bphi^{*(i)} \right)^2\right) ,\\
            \hat{a}_i \mid \mathcal{T} &\sim \mathcal{N}\left( a_i, \frac{\sigma_{\text{DP}}^2 C^2}{\kappa^2 \sum_{t = T^{*}}^{T - 1} \left(\bphi_{t}^{(i)} - \bphi^{*(i)} \right)^2}\right).
        \end{split}
    \end{align}
    From this, we can clearly see that the variance of $\hat{a}_i \mid \mathcal{T}$ is inversely proportional to the sum:
    \[
        \sum_{t = T^{*}}^{T - 1} \left(\bphi_{t}^{(i)} - \bphi^{*(i)} \right)^2.
    \]
    Let $e_{\text{std}} = \underset{{1 \leq i \leq d}}{\max} \sqrt{\text{Var}\left[\hat{a}_i\right]}$, then for all $i$,
    \[
       \frac{\sigma_{\text{DP}}^2 C^2}{\kappa^2\sum_{t = T^{*}}^{T - 1} \left(\bphi_{t}^{(i)} - \bphi^{*(i)} \right)^2} \leq e_{\text{std}}^2,
    \]
    re-arranging, we obtain
    \begin{equation}
        \sum_{t = T^{*}}^{T - 1} \left(\bphi_{t}^{(i)} - \bphi^{*(i)} \right)^2 \geq \frac{\sigma_{\text{DP}}^2 C^2}{\kappa^2 e_{\text{std}}^2}.
    \end{equation}
    Moreover,
    \begin{equation}\label{trace-spread-lower-bound}
        \sum_{t = T^{*}}^{T - 1} \left\lVert \bphi_t - \bphis \right\lVert^2 = \sum_{t = T^{*}}^{T - 1} \sum_{i = 1}^{d} \left(\bphi_{t}^{(i)} - \bphi^{*(i)} \right)^2 \geq \frac{d  \sigma_{\text{DP}}^2 C^2}{\kappa^2 e_{\text{std}}^2}.
    \end{equation}
    
    On the other hand, each $\left\lVert \bphi_t - \bphis \right\rVert$ cannot be too large for the second-order Taylor approximation of $\mathcal{L}\left(\bphi; \bfD\right)$ to be valid. We know that $\mathbf{A}$ is the hessian of $\mathcal{L}\left(\bphi;\bfD\right)$ at $\bphis$, so:
    \[
    \lim_{\bphi \rightarrow \bphis} \frac{\lVert \nabla_{\bphi} \mathcal{L}\left(\bphi\right) - \nabla_{\bphi} \mathcal{L}\left(\bphis\right) - \mathbf{A}(\bphi - \bphis) \rVert}{\lVert \bphi - \bphis \rVert} = 0,
    \]
    since $\nabla_{\bphi} \mathcal{L}\left(\bphis; \bfD\right) = 0$,
    \[
    \lim_{\bphi \rightarrow \bphis} \frac{\lVert \nabla_{\bphi} \mathcal{L}\left(\bphi\right) - \mathbf{A}(\bphi - \bphis) \rVert}{\lVert \bphi - \bphis \rVert} = 0.
    \]
    In other words, for all $e_{\text{tay}} > 0$, there exists $e_{\text{trace}} > 0$ such that
    \[
        0 < \lVert\bphi - \bphis\rVert < e_{\text{trace}} \Longrightarrow \frac{\lVert \nabla \mathcal{L}\left(\bphi\right) - \mathbf{A}(\bphi - \bphis) \rVert}{\lVert \bphi - \bphis \rVert} < e_{\text{tay}}.
    \]
    Let $e_{\text{tay}}$ be as in \Cref{Assumption-taylor-approximation-of-loss}, then $e_{\text{trace}} = r^{*}$, and
    \begin{equation}\label{hessian-limit-trace-inequality}
        0 < \lVert\bphi_t - \bphis\rVert < r^{*},
    \end{equation}
    for all $t = T^{*}, \ldots, T - 1$. This implies that:
    \begin{equation}\label{trace-spread-upper-bound}
            0 < \sum_{t=T^{*}}^{T - 1} \lVert\bphi_t - \bphis\rVert^2 < (T - T^{*})(r^{*})^2.
    \end{equation}
    Combining (\ref{trace-spread-lower-bound}) and (\ref{trace-spread-upper-bound}), we obtain:
    \begin{equation}\label{trace-spread-bounds}
       \frac{d \sigma_{dp}^2 C^2}{\kappa^2 e_{\text{std}}^2} \leq \sum_{t=T^{*}}^{T - 1} \lVert\bphi_t - \bphis\rVert^2 < (T - T^{*}) (r^{*})^2.
    \end{equation}
    Finally, (\ref{trace-spread-bounds}) implies that:
    \begin{equation}
        e_{\text{std}} \times r^{*} \geq \frac{1}{\kappa} \sqrt{\frac{d}{T - T^{*}}}\sigma_{\text{DP}} C.
    \end{equation}
\end{proof}

\section{Proof of Theorem 3}\label{supplemental-section-relation-learning-rate-taylor}
For the second-order Taylor approximation to be accurate, certain values of the learning rate are valid. To show this, we will first approximate the discrete DP-SGD process using a stochastic differential equation (SDE). We provide the following proof for \Cref{stationary-variance-Theorem}.

\begin{proof}
    Let $\mathbf{X}^{\frac{1}{2}}$ denote the square root of a positive semidefinite matrix $\mathbf{X}$. According to the parameter-based model, assuming $\Sigma_{\text{sub}} = \mathbf{0}$,
    \begin{equation} \label{trace-model-near-optimal-1}
        \bphi_{t + 1} = \bphi_{t} - \lambda \kappa \mathbf{A} \left( \bphi_t - \bphi^{*} \right) + \lambda \Sigma_{\text{DP}}^{\frac{1}{2}}\pmb{\eta}_{t + 1},
    \end{equation}
    where $\pmb{\eta}_{t + 1} \sim \mathcal{N}\left(\mathbf{0}, \mathbf{I}_d\right)$, and $\pmb{\eta}_0 = \mathbf{0}$. The random variable $\pmb{\eta}_t$ can be written using the Wiener process $\mathbf{W}_t$:
    \[
        \pmb{\eta}_{t + 1} = \mathbf{W}_{t + 1} - \mathbf{W}_{t} \sim \mathcal{N}\left( \mathbf{0}, \mathbf{I}_d \right),
    \]
    this way, we can re-write equation (\ref{trace-model-near-optimal-1}) as:
    \begin{equation} \label{trace-model-near-optimal-2}
    \Delta \bphi_{t + 1} = - \lambda \kappa \mathbf{A} \left( \bphi_t - \bphi^{*} \right) + \lambda \Sigma_{\text{DP}}^{\frac{1}{2}} \Delta \mathbf{W}_{t + 1}.
    \end{equation}
    
    In continuous time, we can approximate (\ref{trace-model-near-optimal-2}) as:
    \begin{equation} \label{trace-model-near-optimal-3}
        \mathbf{d}\bphi_{t} =  - \lambda \kappa \mathbf{A} \left( \bphi_t - \bphi^{*} \right) \mathbf{d}t +  \lambda \Sigma_{\text{DP}}^{\frac{1}{2}} \mathbf{d}\mathbf{W}_t,
    \end{equation}
    This SDE is a special case of the multi-dimensional Ornstein-Uhlenbeck (OU) process. It is equivalent to the following integral form:
    \begin{equation} \label{trace-model-near-optimal-4}
        \bphi_{t} = \bphi_{0} - \int_{0}^{t} \lambda \kappa \mathbf{A} \left( \bphi_s - \bphi^{*} \right) \mathbf{d}s + \int_{0}^{t} \lambda \Sigma_{\text{DP}}^{\frac{1}{2}}  \mathbf{d}\mathbf{W}_s,
    \end{equation}
    
    To write equation (\ref{trace-model-near-optimal-3}) in the more common form, define $\pmb{\psi}_{t} = (\bphi_t - \bphi^{*})$, then
    \begin{equation} \label{trace-model-near-optimal-5}
        \mathbf{d}\pmb{\psi}_{t} =  - \lambda \kappa \mathbf{A} \pmb{\psi}_{t} \mathbf{d}t +\lambda \Sigma_{\text{DP}}^{\frac{1}{2}} \mathbf{d}\mathbf{W}_t.
    \end{equation}
    from the fact that
    \[
        \pmb{\psi}_{t} - \pmb{\psi}_{0} = \left( \bphi_t - \bphi^{*} \right) - \left( \bphi_0 - \bphi^{*} \right) = \bphi_t - \bphi_0,
    \]
    The solution of equation (\ref{trace-model-near-optimal-5}) is given by,
    \begin{align}
        &\pmb{\psi}_{t} = \pmb{e}^{-\lambda \kappa\mathbf{A} t} \pmb{\psi}_0 + \int_{0}^{t} \lambda \Sigma_{\text{DP}}^{\frac{1}{2}} \pmb{e}^{-\lambda \kappa\mathbf{A} (t - s)} \mathbf{d} \mathbf{W}_s \label{trace-model-near-optimal-6}  \\
        \iff &\bphi_{t} = \bphi^{*} + \pmb{e}^{-\lambda \kappa\mathbf{A} t} \left(\bphi_0 -\bphi^{*}\right)  +  \int_{0}^{t} \lambda \Sigma_{\text{DP}}^{\frac{1}{2}} \pmb{e}^{-\lambda\kappa \mathbf{A} (t - s)} \mathbf{d} \mathbf{W}_s ,\label{trace-model-near-optimal-7}
    \end{align}
    where $\pmb{e}^{-\lambda \kappa \mathbf{A} t}$ is the matrix exponent of $-\lambda \kappa \mathbf{A} t$,
    \[
        \pmb{e}^{-\lambda\kappa \mathbf{A} t} = \sum_{k = 0}^{\infty} \frac{(-\lambda\kappa t)^k}{k!} \mathbf{A}^{k} = \mathbf{I}_d - \lambda \kappa t \mathbf{A} + \frac{\lambda^2 t^2}{2}\kappa^2 \mathbf{A}^2 + \ldots.
    \]
    Note that
    \[
    \boldee\left[ \bphi_{t} \right] = \bphi^{*} + \pmb{e}^{-\lambda \kappa\mathbf{A} t} \left(\bphi_0 -\bphi^{*}\right).
    \]
    We are interested in the stationary distribution of $\bphi_{t}$, so from (\ref{trace-model-near-optimal-6}, \ref{trace-model-near-optimal-7}):
    \[
        \boldee\left[ \bphi_{t} \right] \underset{t \rightarrow \infty}{\longrightarrow} \bphi^{*},
    \]
    \[
        \text{Cov}\left[ \bphi_{t} \right] \underset{t \rightarrow \infty}{\longrightarrow} \Sigma_{\text{sde}},
    \]
    where $\Sigma_{\text{sde}}$ is the solution of the following Lyapunov equation:
    \begin{align}\label{sde-covariance-solution}
        \begin{split}
            &\lambda \kappa\mathbf{A} \Sigma_{\text{sde}} + \lambda\kappa \Sigma_{\text{sde}} \mathbf{A}^\top = \lambda^2 \sigma_{\text{DP}}^2 C^2 \mathbf{I}_d \\
        \iff & \mathbf{A} \Sigma_{\text{sde}} + \Sigma_{\text{sde}} \mathbf{A} = \frac{\lambda}{\kappa}\sigma_{\text{DP}}^2 C^2 \mathbf{I}_d.
        \end{split}
    \end{align}
    Hence, given that $\mathbf{A}$ is constant with respect to $\lambda$ (since it only depends on the $\mathcal{L}$, $\bfD$, and $\bphis$), then $\Sigma_{\text{sde}}$ is proportional to $\lambda$. Now inequality (\ref{hessian-limit-trace-inequality}), implies that:
    \begin{equation}\label{trace-error-sde-inequality}
        (r^{*})^2 > \boldee \left[ \lVert \bphi_{t} - \bphi^{*} \rVert^2 \right] = \text{Tr}\left(\Sigma_{\text{sde}}\right) \propto \frac{\lambda}{\kappa} \sigma_{\text{DP}}^2 C^2.
    \end{equation}
    For the special case when $\mathbf{A}$ is diagonal, then $\Sigma_{\text{sde}}$ is a diagonal matrix since it has to satisfy (\ref{sde-covariance-solution}) and the right hand side is a diagonal matrix.
    
    If $\mathbf{A} = \mathrm{diag} \left(a_1, a_2, \ldots, a_d\right)$, then $\Sigma_{\text{sde}} = \mathrm{diag}\left( \sigma_{\text{sde}_1}^2, \sigma_{\text{sde}_2}^2, \ldots, \sigma_{\text{sde}_d}^2\right)$, and (\ref{sde-covariance-solution}) implies that:
    \[
        \sigma_{\text{sde}_{k}}^2 = \frac{\lambda \sigma_{\text{DP}}^2 C^2}{2 a_k \kappa}
    \]
    or
    \[
        \sigma_{\text{sde}_{k}} = \sqrt{\frac{\lambda}{2 a_k \kappa}} \sigma_{\text{DP}} C.
    \]
    Hence, $\left\lVert \bphi_{t} - \bphi^{*} \right\lVert^2$ has the following stationary mean:
    \[
        \sum_{k = 1}^{d} \frac{\lambda \sigma_{\text{DP}}^2 C^2}{2 a_k \kappa} = \frac{\lambda \sigma_{\text{DP}}^2 C^2}{2 \kappa} \sum_{k = 1}^{d} \frac{1}{a_k}  = \frac{\lambda \sigma_{\text{DP}}^2 C^2}{2 \kappa} \mathbf{Tr}\left( \mathbf{A}^{-1} \right),
    \]
    thus,
    \begin{equation}\label{trace-average-spread}
        \boldee \left[ \lVert \bphi_{t} - \bphi^{*} \lVert^2 \right] = \frac{\lambda \sigma_{\text{DP}}^2 C^2}{2 \kappa}\mathbf{Tr}\left( \mathbf{A}^{-1} \right).
    \end{equation}
    Now inequality (\ref{hessian-limit-trace-inequality}), implies that:
    \[
     (r^{*})^2 > \frac{\lambda \sigma_{\text{DP}}^2 C^2}{2 \kappa}\mathbf{Tr}\left( \mathbf{A}^{-1} \right).
    \]
    Therefore,
    \begin{equation}\label{lambda-inequality-sde-trace}
         \lambda < \frac{2 \kappa (r^{*})^2}{\sigma_{\text{DP}}^2 C^2 \mathbf{Tr}\left( \mathbf{A}^{-1} \right) }.
    \end{equation}
\end{proof}

Inequality (\ref{lambda-inequality-sde-trace}) nor inequality (\ref{trace-error-sde-inequality}) help us choose the learning rate since the relation between them and inequality (\ref{hessian-limit-trace-inequality}) is not a logical equivalence and we typically don't have any information about $\mathbf{A}$ in advance. The reason why we approximated the DP-SGD process using the OU SDE is that it is easier to work with. That said, it is possible to do a similar analysis with the discrete process but with some additional assumptions. To see this,    observe that (\ref{trace-model-near-optimal-1}) can be written as (when $t = T - 1$):
\[
    \bphi_{T} = \left(\mathbf{I}_{d} - \lambda  \kappa \mathbf{A}\right)\bphi_{T - 1} + \lambda \kappa \mathbf{A} \bphis + \lambda  \Sigma_{\text{DP}}^{\frac{1}{2}} \pmb{\eta}_{T},
\]
and that this could be expanded as:
\[
    \bphi_{T} = \left(\mathbf{I}_{d} - \lambda \kappa \mathbf{A}\right)^{T}\bphi_{0} + \lambda \kappa\sum_{t = 0}^{T - 1} \left(\mathbf{I}_{d} - \lambda \kappa \mathbf{A}\right)^{t} \mathbf{A} \bphis + \lambda   \sum_{t = 0}^{T - 1} \left(\mathbf{I}_{d} - \lambda \kappa\mathbf{A}\right)^{t}\Sigma_{\text{DP}}^{\frac{1}{2}}\pmb{\eta}_{T - t}.
\]
Since $\mathbf{A}$ is a symmetric matrix and thus is diagonalizable, there is a diagonal matrix $\mathbf{L}$ and an orthogonal matrix $\mathbf{U}$ such that $\mathbf{U} \mathbf{U}^\top = \mathbf{I}_{d}$ and $\mathbf{A} = \mathbf{U} \mathbf{L} \mathbf{U}^\top$, then $\mathbf{I}_{d} - \lambda \mathbf{L}$ is also a diagonal matrix, Let $l$ be the absolute value of the largest entry in the diagonal of $\mathbf{L}$, if $\lambda < \frac{1}{l\kappa}$, then
\begin{equation}\label{alternative-discrete-analysis-1}
    \left(\mathbf{I}_{d} - \lambda\kappa \mathbf{A}\right)^{T} \bphi_{0}  = \left(\mathbf{I}_{d} - \lambda\kappa \mathbf{U}\mathbf{L}\mathbf{U}^{\top}\right)^{T} \bphi_{0} =\left( \mathbf{U}\left(\mathbf{I}_{d} - \lambda\kappa \mathbf{L}\right)\mathbf{U}^{\top}\right)^{T} \bphi_{0}  =  \mathbf{U}\left(\mathbf{I}_{d} - \lambda \kappa \mathbf{L}\right)^{T} \mathbf{U}^\top \bphi_{0} \overset{T \rightarrow \infty}{\longrightarrow} \pmb{0}.
\end{equation}
Further assume that the diagonal entries of $\mathbf{L}$ are all non-zero so that the matrix $\mathbf{A}$ is invertible, let $\mathbf{Q} = \mathbf{I}_{d} - \lambda \kappa\mathbf{A}$ then, applying the geometric series sum for matrices,
\begin{align}\label{alternative-discrete-analysis-2}
    \begin{split}
        \lambda\kappa \sum_{t = 0}^{T - 1} \mathbf{Q}^{t} \mathbf{A} \bphis &=\lambda\kappa\left(\sum_{t = 0}^{T - 1} \mathbf{Q}^{t} \right) \mathbf{A} \bphis  \\
        &= \lambda\kappa\left(\left(\mathbf{I}_{d} - \mathbf{Q}^{T} \right)\left(\mathbf{I}_{d} - \mathbf{I}_{d} + \lambda \kappa\mathbf{A}\right)^{-1} \right) \mathbf{A} \bphis  \\
         &=\lambda\kappa \left(\left(\mathbf{I}_{d} - \mathbf{Q}^{T} \right)\left(\lambda \kappa \mathbf{A}\right)^{-1} \right) \mathbf{A} \bphis  \\
         &= \frac{\lambda\kappa}{\lambda\kappa} \left(\mathbf{I}_{d} - \mathbf{Q}^{T} \right) \bphis \\
         &= \left(\mathbf{I}_{d}- \left(\mathbf{I}_{d} - \lambda\kappa\mathbf{U} \mathbf{L}\mathbf{U}^{\top}\right)^{T} \right) \bphis \\
         &= \left(\mathbf{I}_{d}- \left(\mathbf{U}\left(\mathbf{I}_{d} - \lambda\kappa \mathbf{L}\right)\mathbf{U}^{\top}\right)^{T} \right) \bphis \\         
         &= \left(\mathbf{I}_{d}- \mathbf{U} \left(\mathbf{I}_{d} - \lambda\kappa\mathbf{L}\right)^{T} \mathbf{U}^{\top} \right) \bphis \\
         &=  \mathbf{U}\left(\mathbf{I}_{d}- \left(\mathbf{I}_{d} - \lambda\kappa \mathbf{L}\right)^{T} \right) \mathbf{U}^\top \bphis \overset{T \rightarrow \infty}{\longrightarrow} \bphis.
    \end{split}
\end{align}
Since $\Sigma_{\text{DP}}$ is $\sigma_{\text{DP}}^2 C^2 \mathbf{I}_d$, it commutes with $ \mathbf{Q}^{t}$, thus from the fact that the noises $\pmb{\eta}_{t}$ are i.i.d.,
\begin{align}\label{alternative-discrete-analysis-3}
    \begin{split}
         \lambda   \sum_{t = 0}^{T - 1} \mathbf{Q}^{t}\Sigma_{\text{DP}}^{\frac{1}{2}}\pmb{\eta}_{T - t} \Bigm| \bphi_{0} &\sim \mathcal{N}\left(\pmb{0}, \lambda^2\sigma_{\text{DP}}^2 C^2  \sum_{t = 0}^{T - 1} \mathbf{Q}^{2t}\right) \\
         &\sim \mathcal{N}\left(\pmb{0}, \lambda^2\sigma_{\text{DP}}^2 C^2  \left(\mathbf{I}_{d} - \mathbf{Q}^{2T} \right)\left(\mathbf{I}_{d} - \mathbf{Q}^2\right)^{-1} \right).
    \end{split}
\end{align}
Hence, from (\ref{alternative-discrete-analysis-1}, (\ref{alternative-discrete-analysis-2}), and (\ref{alternative-discrete-analysis-3}),
\[
    \boldee\left[ \bphi_{T}\right] \overset{T \rightarrow \infty}{\longrightarrow} \bphis,
\]
\[
    \text{Cov}\left[ \bphi_{T}\right] \overset{T \rightarrow \infty}{\longrightarrow}  \lambda^2 \sigma_{\text{DP}}^2 C^2 \left(\mathbf{I}_{d} - \mathbf{Q}^2\right)^{-1}.
\]
Therefore,
\begin{align*}
    \boldee \left[ \lVert \bphi_{T} - \bphis\rVert^2\right] \overset{T \rightarrow \infty}{\longrightarrow} & \lambda^2 \sigma_{\text{DP}}^2 C^2 \text{Tr}\left(\mathbf{U}^{\top} \left(\mathbf{I}_{d} - \left(\mathbf{I}_{d} - \lambda \kappa\mathbf{L}\right)^2\right)^{-1}\mathbf{U} \right) \\
    &= \lambda^2 \sigma_{\text{DP}}^2 C^2 \text{Tr}\left(\left(\mathbf{I}_{d} - \left(\mathbf{I}_{d} - \lambda \kappa\mathbf{L}\right)^2 \right)^{-1}\right) \\
    &= \lambda^2 \sigma_{\text{DP}}^2 C^2 \sum_{i = 1}^{d} \frac{1}{1 - (1 - \lambda\kappa l_i)^2} \\
    &= \lambda^2 \sigma_{\text{DP}}^2 C^2 \sum_{i = 1}^{d} \frac{1}{1 - 1 + 2 \lambda \kappa l_i -  \lambda^2 \kappa^2 l_i^2} \\
    &= \frac{\lambda \sigma_{\text{DP}}^2 C^2}{\kappa} \sum_{i = 1}^{d} \frac{1}{2 l_i  -  \lambda \kappa l_i^2}
\end{align*}
where $l_i$ ($1 \leq i \leq d$) are the diagonal entries of $\mathbf{L}$.

\section{Proof of Theorem 4}\label{supplemental-section-heuristic-learning-rate}
We state the additional assumptions of \Cref{dpsgd-convergence-bound-theorem} again,
\begin{assumption}\label{assumption-bounded-below} $\mathcal{L}\left(\bphi; \mathbf{D}\right)$ is bounded below, i.e. $M = \inf_{\bphi} \mathcal{L}\left(\bphi; \mathbf{D}\right)$ exists.
\end{assumption}

\begin{assumption}\label{assumption-lipchitz-continuous} The per-example losses are Lipschitz continuous $\ell\left(\bphi; \pmb{x}_{i} \right)$ in $\bphi$. Further, assume that the clipping threshold is set to 
\[C = \max_i \sup_{\bphi} \lVert  \nabla_{\bphi} \ell\left(\bphi; \pmb{x}_{i} \right) \rVert .\]
\end{assumption}

\begin{assumption}\label{assumption-loss-lipschitz-smooth}
The loss $\mathcal{L}$ is Lipschitz smooth, i.e. that the gradients $\nabla \mathcal{L}(\pmb{\phi})$ are Lipschitz continuous with Lipschitz constant $H$.
\end{assumption}

Since the per-example losses are assumed to be Lipschitz continuous, then
\[
    \lVert\nabla_{\bphi} \mathcal{L}\left(\bphi; \mathbf{D}\right) \rVert \leq \sum_{i = 1}^{N} \lVert \nabla_{\bphi} \ell\left(\bphi; \pmb{x}_{i} \right)\rVert \leq N C,
\]
so $\mathcal{L}$ is Lipschitz continuous, also
\[
    \sup_{\bphi} \lVert\nabla_{\bphi} \mathcal{L}\left(\bphi; \mathbf{D}\right) \rVert = G \leq N C.
\]

We will denote $\mathcal{L}\left(\bphi; \bfD\right)$ as $\mathcal{L}\left(\bphi\right)$ for short. Also, we want to write the optimization problem as a minimization problem. If the optimization problem is initially a maximization problem, then we change the direction of optimization by re-defining $\mathcal{L}$ as $-\mathcal{L}$. The following lemma is a well-known result; however we prove it within this context.
\begin{lemma}\label{lemma-inequality-lipschitz}
    If $\mathcal{L}\left(\bphi\right)$ is Lipschitz smooth, then for any $\bphi_1, \bphi_2 \in \bphi$, the following inequality holds:
    \[
        \mathcal{L}\left(\bphi_2\right) \leq \mathcal{L}\left(\bphi_1\right) + \nabla \mathcal{L}\left(\bphi_1\right)^{\top}(\bphi_2 - \bphi_1 ) + \frac{H}{2} \lVert\bphi_2 - \bphi_1 \rVert^2.
    \]
\end{lemma}
\begin{proof}
    Define $\gamma(t)$ = $\mathcal{L}\left(\bphi_1 + t (\bphi_2 - \bphi_1)\right)$. By applying the fundamental theorem of calculus,
    \[
        \mathcal{L}\left(\bphi_2\right) - \mathcal{L}\left(\bphi_1\right) = \int_{t = 0}^{1} \gamma'(t) \mathrm{d}t = \int_{t = 0}^{1} \nabla \mathcal{L}\left(\bphi_1 + t (\bphi_2 - \bphi_1)\right)^{\top} (\bphi_2 - \bphi_1 ) \mathrm{d}t.
    \]
    Rearranging the terms,
    \[
        \mathcal{L}\left(\bphi_2\right) = \mathcal{L}\left(\bphi_1\right) + \int_{t = 0}^{1}  \nabla \mathcal{L}\left(\bphi_1 + t (\bphi_2 - \bphi_1)\right)^{\top} (\bphi_2 - \bphi_1 ) \mathrm{d}t,
    \]
    then adding and subtracting $\nabla \mathcal{L}\left( \bphi_1\right)^{\top} (\bphi_2 - \bphi_1 )$ to the right side we obtain
    \begin{align}\label{lemma-inequality-lipschitz-1}
        \begin{split}
            \mathcal{L}\left(\bphi_2\right) = \mathcal{L}\left(\bphi_1\right) + \nabla \mathcal{L}\left( \bphi_1\right)^{\top} (\bphi_2 - \bphi_1 ) + \int_{t = 0}^{1}  \left( \nabla \mathcal{L} \left(\bphi_1 + t (\bphi_2 - \bphi_1)\right) - \nabla \mathcal{L} ( \bphi_1)\right)^{\top} (\bphi_2 - \bphi_1 )\; \mathrm{d}t.
        \end{split}
    \end{align}
    Applying the Cauchy-Schwarz inequality,
    \begin{equation}\label{lemma-inequality-lipschitz-2}
        \left( \nabla \mathcal{L} \left(\bphi_1 + t (\bphi_2 - \bphi_1)\right) - \nabla \mathcal{L} ( \bphi_1)\right)^{\top} (\bphi_2 - \bphi_1 ) \leq \left \lVert \nabla \mathcal{L} \left(\bphi_1 + t (\bphi_2 - \bphi_1)\right) - \nabla \mathcal{L} ( \bphi_1)\right\rVert \lVert \bphi_2 - \bphi_1 \rVert.
    \end{equation}
    From the Lipschitz smoothness condition,
    \begin{equation}\label{lemma-inequality-lipschitz-3}
        \left \lVert \nabla \mathcal{L} \left(\bphi_1 + t (\bphi_2 - \bphi_1)\right) - \nabla \mathcal{L} ( \bphi_1)\right\rVert \leq H t \lVert \bphi_2 - \bphi_1 \rVert.
    \end{equation}
    From (\ref{lemma-inequality-lipschitz-1}), (\ref{lemma-inequality-lipschitz-2}), and (\ref{lemma-inequality-lipschitz-3}),
    \begin{align*}
        \mathcal{L}\left(\bphi_2\right) &\leq \mathcal{L}\left(\bphi_1\right) + \nabla \mathcal{L} ( \bphi_1)^{\top} (\bphi_2 - \bphi_1 ) + \int_{t = 0}^{1} H t \lVert \bphi_2 - \bphi_1 \rVert^2 \mathrm{d}t \\
        &\leq \mathcal{L}\left(\bphi_1\right) + \nabla \mathcal{L} ( \bphi_1)^{\top} (\bphi_2 - \bphi_1 ) + \frac{H}{2} \lVert \bphi_2 - \bphi_1 \rVert^2.
    \end{align*}
\end{proof}

Now the following is the proof of \Cref{dpsgd-convergence-bound-theorem}:
\begin{proof}
From lemma [\ref{lemma-inequality-lipschitz}],
\begin{equation}\label{dp-gd-convergence-inequality-1}
    \mathcal{L}(\bphi_{t + 1}) - \mathcal{L}(\bphi_{t}) \leq \nabla \mathcal{L}(\bphi_{t})^{\top} (\bphi_{t + 1} - \bphi_{t}) + \frac{H}{2}\left\lVert \bphi_{t + 1} - \bphi_t \right\rVert^2.
\end{equation}
Since for any $t$, we have:
\[
\bphi_{t + 1} - \bphi_t = -\lambda \left[ \pmb{g}_{t + 1} + \sigma_{\text{DP}} C \pmb{\eta}_{t + 1} \right],
\]
then we can write:
\begin{align}
    \begin{split}
        \nabla \mathcal{L}(\bphi_{t})^{\top} (\bphi_{t + 1} - \bphi_{t}) &= -\lambda \nabla \mathcal{L}(\bphi_{t})^{\top} \left[ \pmb{g}_{t + 1} + \sigma_{\text{DP}} C \pmb{\eta}_{t + 1} \right] \\
        &= -\lambda \nabla \mathcal{L}(\bphi_{t})^{\top} \pmb{g}_{t + 1} -\lambda \sigma_{\text{DP}} C \nabla \mathcal{L}(\bphi_{t})^{\top} \pmb{\eta}_{t + 1}.
    \end{split}
\end{align}
Observe that $\nabla \mathcal{L}(\bphi_{t})^{\top} \pmb{\eta}_{t + 1} \sim \mathcal{N}\left( \mathbf{0}, \nabla \mathcal{L}(\bphi_{t})^{\top} \mathbf{I}_d \nabla \mathcal{L}(\bphi_{t}) \right) \equiv \mathcal{N}\left( \mathbf{0}, \left\lVert \nabla \mathcal{L}(\bphi_{t}) \right\rVert^2 \mathbf{I}_d \right)$, and $\mathbb{E}\left[ \pmb{g}_{t+1} \mid \pmb{\phi}_t \right] = \kappa \nabla \mathcal{L}(\pmb{\phi}_t)$, and hence,
\begin{equation}\label{dp-gd-convergence-inequality-expectation-1}
    \boldee \left[ \nabla \mathcal{L}(\bphi_{t})^{\top} (\bphi_{t + 1} - \bphi_{t}) \mid \bphi_t \right] = -\lambda \kappa\left\lVert \nabla \mathcal{L}(\bphi_{t}) \right\rVert^2.
\end{equation}
Also since $\left\lVert \bphi_{t + 1} - \bphi_t \right\rVert^2 = \left\langle \bphi_{t + 1} - \bphi_t , \bphi_{t + 1} - \bphi_t  \right\rangle$, then
\begin{align}
    \begin{split}
        \left\lVert \bphi_{t + 1} - \bphi_t \right\rVert^2 &= \left\langle -\lambda \left[\pmb{g}_{t + 1} + \sigma_{\text{DP}} C \pmb{\eta}_{t + 1} \right], -\lambda \left[ \pmb{g}_{t + 1} + \sigma_{\text{DP}} C \pmb{\eta}_{t + 1} \right] \right\rangle \\
        &= \lambda^2 \left\langle \pmb{g}_{t + 1}+ \sigma_{\text{DP}} C \pmb{\eta}_{t + 1},\pmb{g}_{t + 1} + \sigma_{\text{DP}} C \pmb{\eta}_{t + 1} \right\rangle \\
        &= \lambda^2 \left( \left\lVert \pmb{g}_{t + 1} \right\rVert^2 + 2 \sigma_{\text{DP}} C \pmb{g}_{t + 1}^{\top} \pmb{\eta}_{t + 1} + \sigma_{\text{DP}}^2 C^2 \left\lVert \pmb{\eta}_{t + 1} \right\rVert^2 \right) \\
        &= \lambda^2 \left( \left\lVert \pmb{g}_{t + 1} -  \kappa\nabla \mathcal{L}(\bphi_{t}) + \kappa \nabla \mathcal{L}(\bphi_{t}) \right\rVert^2 + 2 \sigma_{\text{DP}} C \pmb{g}_{t + 1}^{\top} \pmb{\eta}_{t + 1} + \sigma_{\text{DP}}^2 C^2 \left\lVert \pmb{\eta}_{t + 1} \right\rVert^2 \right) \\
        &\leq \lambda^2 \left( 2\left\lVert \pmb{g}_{t + 1} - \kappa \nabla \mathcal{L}(\bphi_{t}) \right\rVert^2 + 2\left\lVert \kappa\nabla \mathcal{L}(\bphi_{t}) \right\rVert^2 + 2 \sigma_{\text{DP}} C \pmb{g}_{t + 1}^{\top} \pmb{\eta}_{t + 1} + \sigma_{\text{DP}}^2 C^2 \left\lVert \pmb{\eta}_{t + 1} \right\rVert^2 \right),
    \end{split}
\end{align}
and hence, from assumption \Cref{assumption-loss-lipschitz-smooth}, and the fact that $\mathcal{L}$ is Lipschitz continuous with Lipschitz constant $G$,
\begin{equation}\label{dp-gd-convergence-inequality-expectation-2}
    \boldee\left[ \left\lVert \bphi_{t + 1} - \bphi_t \right\rVert^2 \mid \bphi_t \right] \leq \lambda^2 \left(2e_{\text{sub}}^2 + 2\kappa^2 G^2 + \sigma_{\text{DP}}^2 C^2 d \right).
\end{equation}
Taking the expectation of both sides of the inequality (\ref{dp-gd-convergence-inequality-1}) conditioned on $\bphi_t$, and plugging in both (\ref{dp-gd-convergence-inequality-expectation-1}) and (\ref{dp-gd-convergence-inequality-expectation-2}), we obtain:
\begin{equation}
    \boldee\left[ \mathcal{L}(\bphi_{t + 1}) \mid \bphi_t \right] - \mathcal{L}(\bphi_{t}) \leq -\lambda\kappa \left\lVert \nabla \mathcal{L}(\bphi_{t}) \right\rVert^2 + \frac{H}{2} \left( 2e_{\text{sub}}^2 + 2\kappa^2 G^2+ \sigma_{\text{DP}}^2 C^2 d \right)  \lambda^2.
\end{equation}
Taking the expectation of both sides again, then from the law of total expectation:
\[
    \boldee\left[ \boldee\left[ \mathcal{L}(\bphi_{t + 1}) \mid \bphi_t \right] \right] = \boldee\left[ \mathcal{L}(\bphi_{t + 1}) \right],
\]
\begin{equation}
        \boldee\left[ \mathcal{L}(\bphi_{t + 1}) \right] - \boldee\left[\mathcal{L}(\bphi_{t})\right] \leq -\lambda \kappa\boldee \left[ \left\lVert \nabla \mathcal{L}(\bphi_{t}) \right\rVert^2 \right] + \frac{H}{2} \left( 2e_{\text{sub}}^2 + 2\kappa^2 G^2 + \sigma_{\text{DP}}^2 C^2 d \right)  \lambda^2.
\end{equation}
Re-arranging the terms, we obtain:
\begin{equation}
     \lambda\kappa \boldee \left[ \left\lVert \nabla \mathcal{L}(\bphi_{t}) \right\rVert^2 \right] \leq \boldee\left[\mathcal{L}(\bphi_{t})\right] - \boldee\left[ \mathcal{L}(\bphi_{t + 1}) \right] + \frac{H}{2} \left( 2e_{\text{sub}}^2 + 2\kappa^2 G^2+ \sigma_{\text{DP}}^2 C^2 d \right)  \lambda^2.
\end{equation}
Now by taking the sum of both sides from $t = 0$ to $T - 1$ and telescoping, we obtain:
\begin{equation}
    \sum_{t = 0}^{T - 1} \lambda \kappa \boldee \left[ \left\lVert \nabla \mathcal{L}(\bphi_{t}) \right\rVert^2 \right]  \leq \boldee\left[\mathcal{L}(\bphi_{0})\right] - \boldee\left[ \mathcal{L}(\bphi_{T}) \right] + \frac{H}{2} \left( 2e_{\text{sub}}^2 + 2\kappa^2 G^2+ \sigma_{\text{DP}}^2 C^2 d \right)  \sum_{t = 0}^{T - 1}\lambda^2.
\end{equation}
from assumption \Cref{assumption-bounded-below},
\begin{equation}
    M \leq  \boldee\left[ \mathcal{L}(\bphi_{T}) \right] \Longrightarrow -\boldee\left[ \mathcal{L}(\bphi_{T}) \right] \leq -M,
\end{equation}
and since $\bphi_0$ is a constant,
\begin{equation}\label{dp-gd-convergence-base-inequality}
    \sum_{t = 0}^{T - 1} \lambda \kappa \boldee \left[ \left\lVert \nabla \mathcal{L}(\bphi_{t}) \right\rVert^2 \right]  \leq \mathcal{L}\left(\bphi_{0}\right) - M + \frac{H}{2} \left( 2e_{\text{sub}}^2 + 2\kappa^2 G^2+ \sigma_{\text{DP}}^2 C^2 d \right) \sum_{t = 0}^{T - 1}\lambda^2.
\end{equation}
Since $\underset{0 \leq t \leq T - 1}{\min} \boldee \left[ \left\lVert \nabla \mathcal{L}(\bphi_{t}) \right\rVert^2 \right] \leq \boldee \left[ \left\lVert \nabla \mathcal{L}(\bphi_{t}) \right\rVert^2 \right]$, then the inequality becomes:
\begin{equation}
    \underset{0 \leq t \leq T - 1}{\min}\boldee \left[ \left\lVert \nabla \mathcal{L}(\bphi_{t}) \right\rVert^2 \right]  T \lambda\kappa \leq \mathcal{L}\left(\bphi_{0}\right) - M + \frac{H}{2} \left( 2e_{\text{sub}}^2 + 2\kappa^2 G^2+ \sigma_{\text{DP}}^2 C^2 d \right) T\lambda^2,
\end{equation}
therefore,
\begin{equation}\label{dp-gd-convergence-inequality-2}
\underset{0 \leq t \leq T - 1}{\min}\boldee \left[ \left\lVert \nabla \mathcal{L}(\bphi_{t}) \right\rVert^2 \right] \leq \frac{\mathcal{L}\left(\bphi_{0}\right) - M}{T \lambda \kappa} + \frac{H}{2\kappa} \left( 2e_{\text{sub}}^2 + 2\kappa^2 G^2+ \sigma_{\text{DP}}^2 C^2 d \right) \lambda.
\end{equation}

We can also get the same result as (\ref{dp-gd-convergence-inequality-2}) by dividing both sides of (\ref{dp-gd-convergence-base-inequality}) by $T$ and re-arranging to get:
\begin{equation}\label{dp-gd-convergence-inequality-3}
    \frac{1}{T}\sum_{t = 0}^{T - 1} \boldee \left[ \left\lVert \nabla \mathcal{L}(\bphi_{t}) \right\rVert^2 \right]  \leq \frac{\mathcal{L}\left(\bphi_{0}\right) - M}{T \lambda \kappa} + \frac{H}{2\kappa} \left( 2e_{\text{sub}}^2 + 2\kappa^2 G^2+ \sigma_{\text{DP}}^2 C^2 d \right)  \lambda,
\end{equation}
so we have the same upper bound for the average expected gradient norms and the minimum expected gradient norm. Now we want to find $\lambda$ that minimizes the right-hand side of inequality (\ref{dp-gd-convergence-inequality-2}), in other words, we want $\lambda$ such that
\begin{align*}
&\frac{d}{d \lambda}\left(\frac{\mathcal{L}\left(\bphi_{0}\right) - M}{T \kappa \lambda} + \frac{H}{2 \kappa} \left( 2e_{\text{sub}}^2 + 2\kappa^2 G^2+ \sigma_{\text{DP}}^2 C^2 d \right) \lambda\right) = 0 \\
\iff & -\frac{\mathcal{L}\left(\bphi_{0}\right) - M}{T \lambda^2} + \frac{H}{2} \left( 2e_{\text{sub}}^2 + 2\kappa^2 G^2+ \sigma_{\text{DP}}^2 C^2 d \right) = 0\\
\iff & \frac{H}{2} \left( 2e_{\text{sub}}^2 + 2\kappa^2 G^2+ \sigma_{\text{DP}}^2 C^2 d \right) = \frac{\mathcal{L}\left(\bphi_{0}\right) - M}{T \lambda^2} \\
\iff & \lambda^2 = \frac{2 \left(\mathcal{L}\left(\bphi_{0}\right) - M\right)}{H T \left( 2e_{\text{sub}}^2 + 2\kappa^2 G^2 + \sigma_{\text{DP}}^2 C^2 d \right)}.
\end{align*}
\end{proof}

Define $\lambda_{\text{opt}}$ as:
\[
    \lambda_{\text{opt}} \equiv_{\text{def}} \frac{\sqrt{2 \left(\mathcal{L}\left(\bphi_{0}\right) - M\right)}}{\sqrt{H T \left( 2e_{\text{sub}}^2 + 2\kappa^2 G^2 + \sigma_{\text{DP}}^2 C^2 d \right)}}.
\]
If $e_{\text{sub}}^2$ is insignificant compared to $\kappa^2 G^2 + \sigma_{\text{DP}}^2 C^2 d$, then we can clearly see that
\begin{equation}\label{lambda-optimal-exact}
    \lambda_{\text{opt}} \propto \frac{\sqrt{2}}{\sqrt{H T \left( \kappa^2 G^2 + \sigma_{\text{DP}}^2 C^2 d \right)}}.
\end{equation}
We choose our $\lambda$ according to (\ref{lambda-optimal-exact}); however, if we don't know about $H$ and $G$, then we can either estimate them or simply choose $\lambda_{\text{heur}}$ such that:
\begin{equation}\label{lambda-heuristic}
    \lambda_{\text{heur}} = \frac{\sqrt{2} \lambda_c}{ \sigma_{\text{DP}} C \sqrt{T d} },
\end{equation}
where $\lambda_c > 0$ is a hyper-parameter.

\section{Experiments Additional Details}\label{supplemental-section-experimental-additional-details}
\subsection{Per-Example Loss Function Approximation}
The per-example losses based on the ELBO might be difficult to compute in closed form for some problems, so generally we approximate $\ell$ by sampling from the variational distribution $q_{\text{VI}}$ and then replacing the expected values with averages. So let $\left\{\pmb{\theta}_{i}\right\}_{i = 1}^{N_{\text{VI}}}$, $N_{\text{VI}} > 0$ be samples from $q_{\text{VI}}$, then
\begin{align}\label{approximate-per-example-loss}
    \ell\left(\bphi; \pmb{x} \right) \approx \frac{1}{N_{\text{VI}}}\sum_{i = 1}^{N_{\text{VI}}} \left[\log p\left(\pmb{x} \mid \btheta_{i} \right)\right] - \frac{1}{N_{\text{VI}} N} \sum_{i = 1}^{N_{\text{VI}}}\left[\log \qVI(\btheta_i \mid \bphi) - \log p\left(\btheta_i \right)\right].
\end{align}
Throughout our experiments we use $N_{\text{VI}} = 10$.

If $p\left(\mathbf{D} \mid \pmb{\theta}_{\text{con}}\right)$ and $p(\pmb{\theta}_{\text{con}})$ require some constraints on $\pmb{\theta}$, then we would need to transform $\pmb{\theta}_{\text{con}}$ to an unconstrained domain to work with. From \citep{JMLR:v18:16-107}, we can define any diffeomorphism $\mathcal{U}: \Theta \rightarrow \mathbb{R}^{n}$ such that it transforms $\pmb{\theta}_{\text{con}}$ from the constrained domain to the unconstrained domain $\mathbb{R}^{n}$. Because $\pmb{\theta}$ is unconstrained with respect to the variational distribution, this requires an additional adjustment to (\ref{approximate-per-example-loss}),
\begin{align}\label{approximate-per-example-loss-constrained}
    \begin{split}
        \ell\left(\bphi; \pmb{x} \right) \approx&\frac{1}{N_{\text{VI}}}\sum_{i = 1}^{N_{\text{VI}}} \left[\log p\left(\pmb{x} \mid \mathcal{U}^{-1}\left(\btheta_{i}\right) \right)\right] \\
        &- \frac{1}{N_{\text{VI}} N} \sum_{i = 1}^{N_{\text{VI}}}\left[\log \qVI(\btheta_i \mid \bphi) - \log p\left(\mathcal{U}^{-1}\left(\btheta_i\right) \right)\right] \\
        & + \frac{1}{N_{\text{VI}} N} \sum_{i = 1}^{N_{\text{VI}}} \log \left\lvert \det J_{\mathcal{U}^{-1}} \left(\pmb{\theta}_{i}\right)\right \rvert.
    \end{split}
\end{align}
where $J_{\mathcal{U}^{-1}}\left(\pmb{\theta}_i\right)$ is the Jacobian matrix of $\mathcal{U}^{-1}$ evaluated at $\pmb{\theta}_{i}$ and $\det$ denotes the determinant.

The estimation of $\bphis$ and $\mathbf{A}$ is affected by the choice of the learning-rate, especially the estimation of $\mathbf{A}$ which we fully expand on in \cref{supplemental-section-estimation-of-hessian-learning-rate-relation}. We also establish a heuristic for the learning-rate in \cref{supplemental-section-estimation-of-hessian} which we use extensively in our experiments,
\begin{equation}\lambda_{\text{heur}} = \frac{\sqrt{2} \lambda_{c}}{ \sigma_{\text{DP}} C \sqrt{T d} }.
\end{equation}
It makes it easier to get the learning rate to the right scale and the only decision to be made is the choice of $\lambda_c$ which is a hyper-parameter. Usually, the default value ($\lambda_c = 1$) yields good results. Another value for $\lambda_c$ that we found to yield good results for other models is $\lambda_c = \sqrt{\frac{d}{2}}$.

\subsection{Gradients Preconditioning}\label{supplemental-section-gradients-preconditioning}
We apply a simple preconditioning technique by modifying (\ref{dp-sgd-definition}) so that for a vector $\pmb{\beta} = \left(\pmb{\beta}_{\pmb{\mu}}, \pmb{\beta}_{\pmb{u}}\right)$ which is the concatenation of two vectors $\pmb{\beta}_{\pmb{\mu}}$ and $\pmb{\beta}_{\pmb{u}}$ each of dimension $n$, the update rule becomes,
\begin{align}\label{modified-dp-sgd-align}
    \begin{split}
        \pmb{g}_{t + 1} &= \sum_{i \in \mathcal{B}_{t + 1}} \clip{\pmb{\beta} \odot \nabla_{\pmb{\phi}}\ell(\pmb{\phi}_t; \pmb{x}_i), C}, \\
        \widetilde{\pmb{g}}_{t + 1} &= \frac{1}{\pmb{\beta}} \odot \left[ \pmb{g}_{t + 1}  + \sigma_{\text{DP}} C \pmb{\eta}_{t + 1} \right], \\
        \pmb{\phi}_{t + 1} &= \pmb{\phi}_{t} - \lambda \widetilde{\pmb{g}}_{t + 1},
    \end{split}
\end{align}
where $\odot$ is the element-wise multiplication operation, and $\frac{1}{\pmb{\beta}}$ denotes the vector obtained from the reciprocals of the components of $\pmb{\beta}$. In our experiments, $\pmb{\beta}_{\pmb{\mu}}$ is filled with $1$s and we only choose specific values for $\pmb{\beta}_{\pmb{u}}$. We also need to modify the gradient-based (\ref{gradient-model-trace-approximate}) model so that
\begin{equation}\label{modified-gradient-based-model}
    \tildebg_{t + 1} \mid \bphi_{t}, \mathbf{A}, \bphis
    \sim \mathcal{N}\left(
        \kappa \mathbf{A} \left(\bphi_t - \bphis\right),
        \operatorname{diag}\left(\frac{1}{\pmb{\beta}}\right)
        \left(
            \sigma_{\text{DP}}^2 C^2 \mathbf{I}_d
            + \Sigma_{\text{sub}}
        \right)
        \operatorname{diag}\left(\frac{1}{\pmb{\beta}}\right)
    \right),
\end{equation}
where $\operatorname{diag}(1/\pmb{\beta})$ denotes the diagonal matrix whose entries are the reciprocals of the corresponding components of $\pmb{\beta}$.
Regarding the learning rate $\lambda$, we use our heuristic but multiply the heuristic with $\pmb{\beta}$ to ensure also that all the parameters converge at a similar rate,
\[
\lambda = \lambda_{\text{heur}} \pmb{\beta},
\]
and this makes lambda into a vector so that the product in the last equation in (\ref{modified-dp-sgd-align}), i.e. $\lambda \widetilde{g}_{t + 1}$ becomes an element-wise product $\lambda \odot \widetilde{g}_{t + 1}$.

\subsection{Setting the Priors}\label{supplemental-section-setting-the-priors}
We choose Gaussian distributions as priors for both $\bphis$ and each $v_i$. 
For $\bphis$ we set,
\begin{equation}
    \bphis \sim \mathcal{N}\left(\frac{1}{T - T^{*}} \sum_{t = T^{*}}^{T - 1}\bphi_{t}, \mathbf{I}_{d}\right).
\end{equation}
The prior for each $v_i$ is set based on the MLE estimate for $\mathbf{A} = \mathrm{diag}(v_1, \ldots, v_d)$ in the proof of \Cref{var-of-hessian-trace-trade-off-Theorem},
\begin{align*}
    v_i &\sim \mathcal{N}\left(\mu_{v_i}, \sigma_{v_i}^2\right), \\
    \mu_{v_i} &= \frac{\left\lvert \sum_{t=T^{*}}^{T - 1} \tildebg_{t + 1}^{(i)} \left(\bphi_t^{(i)} - \overline{\bphi}^{(i)}\right) \right \rvert}{\kappa \sum_{t=T^{*}}^{T - 1} \left(\bphi_t^{(i)} - \overline{\bphi}^{(i)}\right)^2}, \\
    \sigma_{v_i}^2 &= \frac{\sigma_{\text{DP}}^2 C^2}{\kappa^2 \left(\pmb{\beta}^{(i)}\right)^2 \sum_{t = T^{*}}^{T - 1} \left(\bphi_{t}^{(i)} - \overline{\bphi}^{(i)} \right)^2}.
\end{align*}
where $\overline{\bphi}$ is the average of the trace after the burn-in index $T^{*}$,
\[
    \overline{\bphi} = \frac{1}{T - T^{*}} \sum_{t = T^{*}}^{T - 1}\bphi_{t}.
\]

\subsection{TARP Evaluation Method}\label{supplemental-section-tarp-evaluation-method}
We can further approximate the average coverages:
\[
    C(\alpha) = \frac{1}{K}\sum_{i = 1}^{K} \boldone_{i, \alpha}\left(\btheta_i\right),
\]
by applying the TARP algorithm from \citep{pmlr-v202-lemos23a}. This is done by sampling from the approximate posterior $\tildep\left(\btheta \mid \bxi \right)$ and assuming that the credible regions are balls centered around $\bthetaRef\left(\bxi_i\right)$. Denote $\bthetaRef\left(\bxi_i\right)$ as $\pmb{\theta}_{\text{ref}(i)}$ and let $\left\{ \widetilde{\btheta}_{j, i} : 1 \leq j \leq N_{\text{tarp}}\right\}$ be samples from the approximate posterior $\tildep\left(\btheta \mid \bxi_{i} \right)$ for each $i$, then
\begin{equation}
        C(\alpha) \approx \frac{1}{K}\sum_{i = 1}^{K} \boldone\left[f_i < 1 - \alpha\right],
\end{equation}
where
\begin{equation}
    f_i = \frac{1}{N_{\text{tarp}}} \sum_{j = 1}^{N_{\text{tarp}}} \boldone\left[d\left( \widetilde{\btheta}_{j, i}, \btheta_{\text{ref}(i)}\right) < d\left( \btheta_{i}, \btheta_{\text{ref}(i)}\right) \right].
\end{equation}
The motivation for why this is a valid approximation is that:
\[
    \btheta_i \in \tildeCoverage\left(\bxi_i, \btheta_{\text{ref}(i)}\right) \iff  d\left( \btheta_{i}, \btheta_{\text{ref}(i)}\right) < r\left(\alpha, \bxi_i\right),
\]
where $r\left(\alpha, \bxi_i\right)$ is the radius of the credible region. Also, this is equivalent to
\[
 \text{Ball}\left(\btheta_{\text{ref}(i)}, d\left( \btheta_{i}, \btheta_{\text{ref}(i)}\right)\right) \subset \tildeCoverage\left(\bxi_i, \btheta_{\text{ref}(i)}\right),
\]
i.e. the ball centered around $\btheta_{\text{ref}(i)}$ with radius $d\left( \btheta_{i}, \btheta_{\text{ref}(i)}\right)$ is contained in the credible region. Therefore, it is also equivalent to
\[
    \int_{\Theta} \boldone \left[ \widetilde{\btheta} \in \text{Ball}\left(\btheta_{\text{ref}(i)}, d\left( \btheta_{i}, \btheta_{\text{ref}(i)}\right)\right)\right] \tildep \left(\widetilde{\btheta} \mid \bxi_{i} \right) < 1 - \alpha.
\]
This integral can be approximated by sampling from $\tildep\left(\widetilde{\btheta} \mid \bxi_{i} \right)$ and then using these samples to calculate the average of $\boldone \left[ \widetilde{\btheta} \in \text{Ball}\left(\btheta_{\text{ref}(i)}, d\left( \btheta_{i}, \btheta_{\text{ref}(i)}\right)\right)\right]$. Also,
\begin{align*}
    & \widetilde{\btheta} \in \text{Ball}\left(\btheta_{\text{ref}(i)}, d\left( \btheta_{i}, \btheta_{\text{ref}(i)}\right)\right) \\
    \iff & \boldone \left[ d\left(\widetilde{\btheta} ,\btheta_{\text{ref}(i)}  \right) <  d\left( \btheta_{i}, \btheta_{\text{ref}(i)}\right)  \right],
\end{align*}
and thus the approximation holds. One important thing to consider here, is that if $\pmb{\theta} \in \Theta$ is constrained, we must have a diffeomorphism $\mathcal{U}: \Theta \rightarrow \mathbb{R}^{n}$ that transforms the prior samples $\pmb{\theta}_{i}$ to the unconstrained space $\mathcal{U}\left(\pmb{\theta}_{i}\right)$. Since $\mathcal{U}$ is a diffeomorphism and thus both injective and surjective, it must preserve the relation $\in$.
Hence,
\[
    \pmb{\theta} \in X \iff \mathcal{U}\left(\pmb{\theta}\right) \in \mathcal{U}\left(X\right)
\]
similarly if $\pmb{\theta}' \in \mathbb{R}^{n}$, then
\[
    \pmb{\theta}' \in Y \iff \mathcal{U}^{-1}\left(\pmb{\theta}'\right) \in \mathcal{U}^{-1}\left(Y\right)
\]

According to Theorem 3 of \citep{pmlr-v202-lemos23a}, if $\widetilde{p}\left(\pmb{\theta} \mid \mathcal{T} \right)$ had the correct coverages, then the marginals $\widetilde{p}\left(\pmb{\theta}^{(i)} \mid \mathcal{T} \right)$ should also have correct coverages. We can also use the TARP algorithm to calculate the coverages for each dimension of the parameters vector $\pmb{\theta}$, i.e., the marginal coverages $\widetilde{p}\left(\pmb{\theta}^{(i)} \mid \mathcal{T}\right)$ where $\pmb{v}^{(i)}$ denotes the $i$th component of a vector $\pmb{v}$.

\subsection{Exponential Families Experiment Details}\label{supplemental-section-exponential-families-constrained-optimization}
For the constrained optimization problem of \ref{expfam-1}, $\pmb{\theta}_{\text{con}} > 0$, so we define
\begin{equation}\label{softplus-conditioning}
\mathcal{U}_1\left(\pmb{\theta}_{\text{con}}\right) = \log(e^{\pmb{\theta}_{\text{con}}} - 1) \implies \mathcal{U}_1^{-1}\left(\pmb{\theta}\right) = \text{softplus}\left(\pmb{\theta}\right).
\end{equation}
For Model \ref{expfam-2}, $\pmb{\theta}_{\text{con}} \in (0, 1)$, so we define
\[
\mathcal{U}_2\left(\pmb{\theta}_{\text{con}}\right) = \log \left(\frac{\pmb{\theta}_{\text{con}}}{1 - \pmb{\theta}_{\text{con}}}\right) \implies \mathcal{U}_2^{-1}\left(\pmb{\theta}\right) = \frac{1}{1 + e^{-\pmb{\theta}}}.
\]
Finally, for Model \ref{expfam-3}, let $\pmb{\theta}_{\text{con}} = \left(\theta_1, \theta_2, \theta_3\right)$, then $\theta_1 + \theta_2 + \theta_3 = 1$ so we only have two degrees of freedom and can let $\theta_3 = 1 - \theta_1 - \theta_2$. We define,
\[
\mathcal{U}_3\left(\pmb{\theta}_{\text{con}}\right) = \left(\log\left(\frac{\theta_1}{\theta_3}\right), \log\left(\frac{\theta_2}{\theta_3}\right)\right).
\]
For the inverse, let $\pmb{\theta} = \left(\theta_1', \theta_2'\right)$, then
\[
\mathcal{U}_3^{-1}\left(\pmb{\theta}\right) = \text{softmax}\left(\theta_1', \theta_2', 0\right).
\]

\subsection{UCI Adult Experiment Details}\label{uci-adult-experiment-details}
For data pre-processing, we removed the columns ``education-num'', ``native-country'', and ``relationship'', and converted categorical values to one-hot encoded vectors. The continuous values were standardized. Moreover, we evaluated our method for $\epsilon \in \{0.1, 0.3, 1.0\}$. Across these different values of $\epsilon$, we used the same number of iterations $T = 10^{4}$ and a sampling rate of $\kappa = 0.1$.

\subsection{NA-DPVI Algorithm}\label{supplemental-section-algorithm-blocks}
\algdef{SE}[SUBALG]{Indent}{EndIndent}{}{\algorithmicend\ }%
\algtext*{Indent}
\algtext*{EndIndent}

The variables that are used in all the algorithms are found in \cref{alg:global-variables}. The DPVI algorithm that we used can be found in \cref{alg:dpvi-algorithm} from which the parameter and gradient traces ($\mathcal{T}$, $\widetilde{\mathcal{G}}$) are obtained. Finally, our NA-DPVI algorithm can be found in \cref{alg:na-dpvi}.

The time complexity of DPVI (\cref{alg:dpvi-algorithm}) is $\mathcal{O}(T\times B \times d)$ where $T$ is the number of steps of DP-SGD, $B = \kappa N$ is the expected batch size ($\kappa$ is the subsampling ratio and $N$ is the number of samples), and $d$ is dimensionality of $\pmb{\phi}$ (i.e. parameters). 

The time complexity of NA-DPVI (\cref{alg:na-dpvi}) has an additional cost of sampling $M$ from the posterior $p\left(\pmb{\phi}^{*}, \mathbf{A}, \Sigma_{\text{sub}} \mid \widetilde{\mathcal{G}}, \mathcal{T}\right)$, which depends on the method of approximating the noise-aware posterior \cref{approximate-noise-aware-posterior} (e.g., NUTS and Laplace). The complexity of the noise-aware posterior approximation method is a function $M$, $d$, and $T$ since the approximation method takes as its input the trace $\mathcal{T}$ and perturbed gradients $\widetilde{\mathcal{G}}$.

\begin{algorithm}
    \begin{algorithmic}
        \State $\left(\pmb{x}_i\right)_{i = 1}^{N}:$ data;
        \State $(\epsilon, \delta):$ DP privacy level;
        \State $C:$ clipping threshold;
        \State $T :$ training iterations;
        \State $\lambda :$ learning rate;
        \State $\kappa :$ Poisson sampling rate;
        \State $\pmb{\beta} :$ gradients preconditioning vector; \Comment{see \cref{supplemental-section-gradients-preconditioning}, for more details about preconditioning.}
        \State $\pmb{\phi}_0 :$ initial values for $\pmb{\phi}$;
        \State $\sigma_{\text{DP}} :$ PRVAccountant$(\epsilon, \delta, T, \kappa)$; \Comment{see \citet{PRVAccountantGopi2021}.}
    \end{algorithmic}
    \caption{\textbf{Global variables}}
    \label{alg:global-variables}
\end{algorithm}
\begin{algorithm}
    \begin{algorithmic}[1]
        \For {$t = 0,\ldots,T - 1$}
            \State $\pmb{g}_{t + 1} \gets \sum_{i \in \mathcal{B}_{t + 1}} \clip{\pmb{\beta} \odot \nabla_{\pmb{\phi}}\ell(\pmb{\phi}_t; \pmb{x}_i), C}$; \Comment{see \cref{supplemental-section-gradients-preconditioning}, for more details about preconditioning.}
            \State Sample $\pmb{\eta}_{t + 1} \sim \mathcal{N}(0, \mathbf{I}_d)$;
            \State $\widetilde{\pmb{g}}_{t + 1} \gets \frac{1}{\pmb{\beta}} \odot \left[\pmb{g}_{t + 1}  + \sigma_{\text{DP}} C \pmb{\eta}_{t + 1}\right]$;
            \State $\pmb{\phi}_{t + 1} \gets \pmb{\phi}_{t} - \lambda \widetilde{\pmb{g}}_{t + 1}$;
        \EndFor
        \State \textbf{return} $\mathcal{T} = \left( \pmb{\phi}_t \right)_{t = 0}^{T}$, $\widetilde{\mathcal{G}} = \left(\widetilde{\pmb{g}_t}\right)_{t = 1}^{T}$;
    \end{algorithmic}
    \caption{DPVI}
    \label{alg:dpvi-algorithm}
\end{algorithm}
\begin{algorithm}
    \begin{algorithmic}[1]
    \State \textbf{Input}:
        \Indent 
            \State $M:$ number of samples;
            \State $\pmb{\mu}_{\pmb{\phi}^*} : \pmb{\phi}^*$ prior distribution mean;
            \State $\mathbf{\Sigma}_{\pmb{\phi}^*} : \pmb{\phi}^*$ prior distribution covariance matrix;
            \State $\pmb{\mu}_{\pmb{v}} : \pmb{v}$ prior distribution mean;
            \State $\mathbf{\Sigma}_{\pmb{v}} : \pmb{v}$ prior distribution covariance matrix;
            \State $p_{\text{sub}}:$ prior for $\Sigma_{\text{sub}}$ supported on positive semidefinite matrices, or a point mass at $\mathbf{0}$;
            \State $\mathcal{T}, \widetilde{\mathcal{G}} :$ DPVI output;
        \EndIndent
        
        \State \textbf{model-definition} $p\left(\widetilde{\mathcal{G}} \mid \mathcal{T}, \pmb{\phi}^{*}, \mathbf{A}, \Sigma_{\text{sub}}\right)$
        \Indent 
        \State $\pmb{\phi}^* \sim \mathcal{N}(\pmb{\mu}_{\pmb{\phi}^*}, \mathbf{\Sigma}_{\pmb{\phi}^*})$;
        \State $\pmb{v} \sim \mathcal{N}(\pmb{\mu}_{\pmb{v}}, \mathbf{\Sigma}_{\pmb{v}})$;
        \State $\mathbf{A} = \operatorname{diag}\left(\operatorname{softplus}(\pmb{v})\right)$; \Comment{softplus is applied element-wise.}
        \State $\Sigma_{\text{sub}} \sim p_{\text{sub}}$; \Comment{or fix $\Sigma_{\text{sub}} = \mathbf{0}$ in the simplified model.}
        \State $\tildebg_{t + 1} \mid \bphi_{t}, \mathbf{A}, \bphis, \Sigma_{\text{sub}} \sim \mathcal{N}\left(
        \kappa \mathbf{A} \left(\bphi_t - \bphis\right),
        \operatorname{diag}\left(\frac{1}{\pmb{\beta}}\right)
        \left(\sigma_{\text{DP}}^2 C^2 \mathbf{I}_d + \Sigma_{\text{sub}}\right)
        \operatorname{diag}\left(\frac{1}{\pmb{\beta}}\right)
        \right)$; \Comment{see \cref{modified-gradient-based-model}.}
        \EndIndent 
        \State \textbf{end model-definition}
        \State Sample $\left(\pmb{\phi}^*_{i}, \mathbf{A}_{i}, \Sigma_{\text{sub}, i}\right)_{i = 1}^{M} \sim p\left(\pmb{\phi}^{*}, \mathbf{A}, \Sigma_{\text{sub}} \mid \widetilde{\mathcal{G}}, \mathcal{T}\right)$; \Comment{sampling using any approximate inference method.}
        \State \textbf{return} $\widetilde{p}(\pmb{\theta} \mid \mathcal{T}) = \frac{1}{M} \sum_{i = 1}^{M} q_{\text{VI}}(\btheta; \bphis_{i})$; \Comment{approximate noise-aware posterior mixture model.}
    \end{algorithmic}
    \caption{NA-DPVI}
    \label{alg:na-dpvi}
\end{algorithm}

\subsection{Computational Resources}\label{supplemental-computational-resources}
We used the computational resources offered by CSC – IT Center for Science, Finland. In particular, the Puhti supercomputer was used to run experiments on CPU nodes in parallel. Each Puhti node is equipped with two Intel Xeon processors, code name Cascade Lake, with 20 cores each running at 2.1 GHz. 

\clearpage
\onecolumn
\begingroup
\definecolor{navy}{HTML}{16324F}
\definecolor{pale}{HTML}{F0F4F8}
\setlength{\parindent}{0pt}
\setlength{\tabcolsep}{5pt}
\newcolumntype{L}[1]{>{\raggedright\arraybackslash}p{#1}}
\newcommand{\E}{\mathbb E}
\newcommand{\V}{\operatorname{Var}}
\newcommand{\ind}{\mathbf1}
\newcommand{\loc}[2]{\textbf{p.~#1}\enspace #2}

\section*{Typographical errors corrected in this arXiv version}\label{app:corrections}
\vspace{3pt}
{\bfseries Noise-Aware Differentially Private Variational Inference}\par
Page and equation references refer to the published camera-ready version.\par
\vspace{3pt}
\vspace{5pt}
\fontsize{9}{11}\selectfont
\renewcommand{\arraystretch}{1.14}
\setlength{\extrarowheight}{1pt}
\setlength{\LTpre}{0pt}
\setlength{\LTpost}{0pt}
{\color{navy}\bfseries Main paper}\par\vspace{3pt}
\rowcolors{2}{pale}{white}
\begin{longtable}{L{40mm}L{53mm}L{\dimexpr\textwidth-93mm-6\tabcolsep\relax}}
\rowcolor{navy}\color{white}\bfseries Published location & \color{white}\bfseries Original & \color{white}\bfseries Corrected\\
\endfirsthead
\rowcolor{navy}\color{white}\bfseries Published location & \color{white}\bfseries Original & \color{white}\bfseries Corrected\\
\endhead
\multicolumn{3}{r}{\footnotesize Continued on next page}\\
\endfoot
\endlastfoot
\loc{3}{\S2.3, Def.~2; before (6)} & $S\subset\operatorname{im}(A)$ & $S\subset\operatorname{Range}(\mathcal A)$\\
\loc{4}{\S3.1, (14), RHS} & $\ind_{\widetilde{\mathcal C}_\alpha(\xi)}(\theta)p(\theta,\xi)$ & $\E_{\theta,\xi\sim p(\theta,\xi)}[\ind_{\widetilde{\mathcal C}_\alpha(\xi)}(\theta)]$\\
\loc{4}{\S3.3, (16), (17)} & $i\in\mathcal B_t$ & $i\in\mathcal B_{t+1}$\\
\loc{5}{Thm.~1; after (26)} & $e_{\rm approx}^2=e_{\rm sub}^2+(\kappa e_{\rm tay}r^*)^2$ & $e_{\rm approx}^2=2e_{\rm sub}^2+2(\kappa e_{\rm tay}r^*)^2$\\
\loc{6}{\S4.1, Hessian paragraph} & Positive Hessian; ELBO sign implicit & Defines the minimized loss as the negative ELBO\\
\loc{8}{\S4.4, (29), integrand} & $p(y=1\mid\theta)$ & $p(y=1\mid\theta,x)$\\
\loc{6}{\S4.2, last sentence} & Refers to marginal coverage plots which do not exist. & Removed.\\
\loc{9}{After \S5; before acknowledgments} & No code-availability paragraph & Adds the link to the code: \url{https://github.com/DPBayes/noise-aware-dpvi}.\\
\end{longtable}\par
\vspace{7pt}
{\color{navy}\bfseries Appendix}\par\vspace{3pt}
\rowcolors{2}{pale}{white}
\begin{longtable}{L{40mm}L{53mm}L{\dimexpr\textwidth-93mm-6\tabcolsep\relax}}
\rowcolor{navy}\color{white}\bfseries Published location & \color{white}\bfseries Original & \color{white}\bfseries Corrected\\
\endfirsthead
\rowcolor{navy}\color{white}\bfseries Published location & \color{white}\bfseries Original & \color{white}\bfseries Corrected\\
\endhead
\multicolumn{3}{r}{\footnotesize Continued on next page}\\
\endfoot
\endlastfoot
\loc{12}{App.~B; before (A3)} & Squared-norm bounds omit factors of 2 ($\|a + b\|^2 \leq  \|a\|^2 + \|b\|^2$) & Coefficient 2 on both terms ($\|a + b\|^2 \leq  2\|a\|^2 + 2\|b\|^2$)\\
\loc{12}{App.~C; after (A5)} & Marginalizes out $\widehat\phi^*_{\rm map}$ & Marginalizes out $V$, retaining the marginal distribution of $\phi^*$.\\
\loc{13}{(A9) and following prose} & $\max_i\V[\hat a_i]$ & $\max_i\sqrt{\V[\hat a_i]}$\\
\loc{14}{App.~D, (A11)} & $e_{\rm sub}^2+\kappa^2G^2+\sigma_{\rm DP}^2C^2d$ & $2e_{\rm sub}^2+2\kappa^2G^2+\sigma_{\rm DP}^2C^2d$\\
\loc{14}{App.~E, opening} & $\mathcal L(\phi^*;D)=0$ & $\nabla\mathcal L(\phi^*;D)=0$\\
\loc{15}{App.~E, (A16), lines 1--3} & $\widetilde g_t^{(i)}$ & $\widetilde g_{t+1}^{(i)}$\\
\loc{15}{Between (A16), (A17)} & $e_{\rm std}=\max_i\V[\hat a_i]$ & $e_{\rm std}=\max_i\sqrt{\V[\hat a_i]}$ this is the max std not the variance \\
\loc{17}{App.~F; after (A29)} & $\sum_{k=0}^{\infty}\frac{(-1)^k\lambda^kt^k}{k}A^k$ & $\sum_{k=0}^{\infty}\frac{(-\lambda\kappa t)^k}{k!}A^k$\\
\loc{20}{App.~G, (A43), last line} & Squared-norm bound omits factor of 2 ($\|a + b\|^2 \leq  \|a\|^2 + \|b\|^2$) & Coefficient 2 on both terms ($\|a + b\|^2 \leq  2\|a\|^2 + 2\|b\|^2$)\\
\loc{20--21}{(A44)--(A48), (A50)--(A53); $\lambda_{\rm opt}$ derivation} & $e_{\rm sub}^2+\kappa^2G^2+\sigma_{\rm DP}^2C^2d$ & $2e_{\rm sub}^2+2\kappa^2G^2+\sigma_{\rm DP}^2C^2d$ throughout\\
\loc{21}{App.~G; before (A48)} & $R-1$; ``both side'' & $T-1$; ``both sides''\\
\loc{18}{App.~F; expansion before (A34), and in (A36)} & Noise index $\eta_{t+1}$  & Uses $\eta_{T-t}$ in both sums.\\
\loc{19}{App.~F; unnumbered variance derivation} & Denominator $2l_i-\lambda l_i^2$ & Denominator $2l_i-\lambda\kappa l_i^2$.\\
\loc{22}{App.~H.1, (A57), final term} & $\lvert\det J_{\mathcal U^{-1}}(\theta_i)\rvert$ & $\log\lvert\det J_{\mathcal U^{-1}}(\theta_i)\rvert$.\\
\loc{23,26}{App.~H.2, (A60); Algorithm 4 likelihood} & Covariance scaled as $(1/\beta)\odot S$ & Uses $DSD$, where $D=\operatorname{diag}(1/\beta)$ and $S=\sigma_{\rm DP}^2 C^2 I_d+\Sigma_{\rm sub}$.\\
\loc{23}{App.~H.3, prior scale formula} & Left-hand side $\sigma_{v_i}$ & Left-hand side $\sigma_{v_i}^2$; the expression specifies variance, not standard deviation.\\
\loc{23}{App.~H.3, after prior formulas} & Burn-in index $\phi^*$ & Burn-in index $T^*$.\\
\loc{25}{App.~H.5, forward $\mathcal U_3$} & $\bigl(\log(\theta_1/\theta_3),\log(\theta_2/\theta_3),0\bigr)$ & $\bigl(\log(\theta_1/\theta_3),\log(\theta_2/\theta_3)\bigr)$ \\
\loc{25}{App.~H.6, preprocessing} & Continuous features ``normalized to be within $(0,1)$'' & Describes the continuous features as standardized.\\
\loc{25}{Algorithm 3, line 1} & $t=1,\ldots,T-1$ & $t=0,\ldots,T-1$, starting from the initialized $\phi_0$.\\
\loc{26}{Algorithm 4, inputs and Hessian prior} & Gaussian prior directly on $A$ & As described in Section 4, Gaussian prior on $v$: $v\sim\mathcal N(\mu_v,\Sigma_v)$; defines $A=\operatorname{diag}(\operatorname{softplus}(v))$ element-wise and updates prior inputs accordingly.\\
\loc{26}{Algorithm 4, subsampling prior} & Unrestricted Gaussian prior directly on $\Sigma_{\rm sub}$ & Uses a specified prior $p_{\rm sub}$ supported on positive semidefinite matrices, or as discussed in Section 3, fixes $\Sigma_{\rm sub}=0$ via a point mass at zero.\\
\end{longtable}\par
\vspace{5pt}
The typographical errors corrected above do not affect the results or conclusions of the paper.\par

\endgroup

\end{document}